\documentclass[a4paper]{cas-sc}

\usepackage[authoryear,longnamesfirst]{natbib}

\usepackage{orcidlink}
\usepackage{graphicx}
\usepackage{amsmath}
\usepackage{amssymb}
\usepackage{amsthm}
\newtheorem{theorem}{Theorem}
\newtheorem{lemma}{Lemma}
\newtheorem{proposition}{Proposition}
\newtheorem{corollary}{Corollary}

\newtheorem{remark}{Remark}
\usepackage{soul}
\usepackage{afterpage}

\usepackage{booktabs}
\usepackage{algorithm}
\usepackage{algorithmicx}
\usepackage{algpseudocode} 
\usepackage{booktabs}
\usepackage{url}
\usepackage{multirow}
\usepackage{xcolor}
 
\usepackage{changes}
\usepackage{wrapfig}
\usepackage{subcaption}
\usepackage{amsmath}
\usepackage{verbatim}
\usepackage{float}
\usepackage{placeins}

\usepackage{hyperref}
\newcommand{\net}{AEODE\ }

\def\tsc#1{\csdef{#1}{\textsc{\lowercase{#1}}\xspace}}
\tsc{WGM}
\tsc{QE}

\begin{document}
\let\WriteBookmarks\relax
\def\floatpagepagefraction{1}
\def\textpagefraction{.001}

\shorttitle{MCMC Informed Neural Emulators}    

\shortauthors{}  

\title [mode = title]{MCMC Informed Neural Emulators for Uncertainty Quantification in Dynamical Systems}  



%
%

\author[1]{Heikki Haario}







\affiliation[1]{organization={Lappeenranta-Lahti University of Technology LUT}, 
            city={Lahti},
            country={Finland}
}

%
%

\author[1]{Zhi-Song Liu}






%
%

\author[2]{Martin Simon}





\affiliation[2]{organization={Frankfurt University of Applied Sciences},
            city={Frankfurt am Main},
            country={Germany}
            }

\affiliation[3]{organization={University of Huddersfield},
            city={Huddersfield},
            country={United Kingdom}
}

%
%

\author[2,3]{Hendrik Weichel\corref{cor1}}


\ead{hendrik.weichel@fra-uas.com}



\cortext[cor1]{Corresponding author}


\begin{abstract}
Neural networks are a commonly used approach to replace physical models with computationally cheap surrogates. Parametric uncertainty quantification can be included in training, assuming that an accurate prior distribution of the model parameters is available. Here we study the common opposite situation, where direct screening or random sampling of model parameters leads to exhaustive training times and evaluations at unphysical parameter values. Our solution is to decouple uncertainty quantification from network architecture. Instead of sampling network weights, we introduce the model-parameter distribution as an input to network training via Markov chain Monte Carlo (MCMC). In this way, the surrogate achieves the same uncertainty quantification as the underlying physical model, but with substantially reduced computation time. The approach is fully agnostic with respect to the neural network choice. In our examples, we present a quantile emulator for prediction and a novel autoencoder-based ODE network emulator that can flexibly estimate different trajectory paths corresponding to different ODE model parameters. Moreover, we present a mathematical analysis that provides a transparent way to relate potential performance loss to measurable distribution mismatch.
\end{abstract}


\maketitle

\section{Introduction}
\label{sec:intro}
Neural networks are increasingly used as emulators, i.e., surrogates for computationally intensive simulation models in physics and chemistry. Beyond accelerating simulations, they can also incorporate complex features that are difficult to encode in traditional solvers. A prominent example is \emph{Physics-Informed Neural Networks} (PINNs), which embed physical laws into the training objective and can serve as fast surrogates using precomputed data~\cite{raissi2019physics,mpinn,stiff-pinn,aurora}. PINNs have been applied to forward and inverse problems in chemical kinetics ~\cite{air_1,air_2,air_3,air_4}, physics ~\cite{physics_1,physics_2,physics_3}, and finance~\cite{finance_1,finance_2,finance_3,finance_4}. Despite these successes, integrating principled \emph{uncertainty quantification} (UQ) into neural emulators remains challenging. A standard approach to UQ is via \emph{Bayesian Neural Networks} (BNNs) ~\cite{bnn,bnn_2}, which treat network weights as random variables and approximate posterior uncertainty using methods such as MCMC, variational inference, or Monte Carlo dropout. While powerful, BNNs can be difficult to scale: the weight space is high-dimensional, selecting informative priors is non-trivial, and approximate posteriors may yield miscalibrated uncertainty, especially in regimes with limited data or model misspecification.

In this work, we propose an alternative, modular approach to Bayesian UQ for neural emulation: the \emph{MCMC Informed Neural Emulator} (MINE). The key idea is to \emph{decouple} Bayesian inference from function approximation. We first perform posterior inference \emph{offline} using MCMC on the original simulator (or any black-box/legacy model) to obtain samples from the parameter posterior \(p(\boldsymbol{\theta}\mid \boldsymbol{y})\). We then train deterministic neural surrogates on \emph{posterior-informed} data, i.e., input--output pairs concentrated in regions of parameter space consistent with the observed data. This approach, related to recent work on training-distribution selection~\cite{guerra2025learning}, directly targets the quantity of interest in Bayesian prediction, without requiring Bayesian architectures at inference time. Concretely, we provide two complementary surrogate components: a quantile (interval) emulator and a forward emulator for operator learning. The two components serve different deployment needs: the quantile emulator provides immediate interval estimates which is particularly useful when deterministic, low-latency uncertainty summaries suffice, while the forward emulator enables fast posterior predictive sampling which is useful for downstream stochastic pipelines.

This paradigm exploits a practical advantage of Bayesian posterior sampling: MCMC concentrates computation on plausible parameter regions instead of exhaustively covering the full parameter space. To illustrate the scale, consider a simulator with a runtime of one second and 20 parameters, each discretized into 10 values. A brute-force grid would require \(10^{20}\) simulator evaluations, i.e.,\ \(10^{20}\) seconds, which is approximately \(3\times 10^{12}\) years. In contrast, a posterior chain with \(10^5\) simulator evaluations takes on the order of a day (about 28 hours at one second per evaluation). The MINE paradigm leverages such posterior samples to reduce the data-generation burden dramatically, while focusing learning on the statistically relevant regions.
The MINE paradigm is \emph{conceptually distinct} from embedding uncertainty directly into the neural network. Unlike BNNs, which attempt to represent uncertainty through probabilistic weight distributions and require posterior approximation in weight space, MINE externalizes uncertainty through Bayesian parameter inference performed offline. Where, e.g., Bayesian PINNs \cite{yang2021b} place Bayesian inference inside the surrogate by treating network weights probabilistically and estimating their posterior, we instead run MCMC offline on the simulator’s physical parameters and then train deterministic emulators on posterior-informed data, thus avoiding weight-space sampling at inference time. This approach is complementary to surrogate-assisted inference methods such as using surrogates to accelerate MCMC (assuming that draws from a prior are possible): here, we assume posterior sampling can be performed offline, with minimal prior knowledge such as positivity constraints only, and then amortize Bayesian prediction through deterministic surrogates trained on posterior-informed data. The MINE paradigm can be understood independently from architectural families such as Neural Ordinary Differential Equations (Neural ODEs)~\cite{ode_1,ode_2,neuralode2,ode_4,ode_5,ode_6}, even though it may utilize these architectures in its realizations. Neural ODEs model continuous-time dynamics via neural-parameterized vector fields and have been used, for example, to emulate chemical reaction systems~\cite{chemiode,chemode_2}. However, standard Neural ODE formulations do not provide Bayesian UQ natively; Bayesian extensions (e.g., Bayesian Neural ODEs or neural SDE variants) incorporate uncertainty internally but typically require sampling at inference time and can introduce additional computational overhead. In contrast, the MINE framework achieves UQ through offline posterior sampling and trains deterministic emulators on posterior-informed data. The forward emulator realization may use a Neural ODE/neural operator backbone, but the UQ mechanism is \emph{architecture-agnostic} and lives outside the network. Relative to PINNs, the MINE framework is generally \emph{equation-free} in training and is particularly suitable when governing equations are unknown, non-differentiable, or when one must interface with black-box simulators. Nevertheless, physics-informed loss terms can be incorporated when available. Finally, compared to Statistics-Informed Neural Networks (SINNs)~\cite{sinn_1,sinn_2}, which aim to reproduce full statistical properties of stochastic trajectories, the MINE paradigm targets Bayesian predictive uncertainty for quantities of interest via posterior-informed training data and produces deterministic outputs at inference time (either conditional on sampled \(\boldsymbol{\theta}\) for posterior draws, or directly as quantiles for intervals).

A central modeling consideration is \emph{what uncertainty is represented}. The MINE framework directly captures \emph{posterior/parameter uncertainty} through \(p(\boldsymbol{\theta}\mid \boldsymbol{y})\) and its propagation to predictions. In addition, any observation noise and stochasticity encoded in the likelihood \(p(\boldsymbol{y}\mid \boldsymbol{\theta})\) can be reflected in the posterior predictive distribution. Like any surrogate approach, MINE also introduces \emph{emulator error}; throughout the paper we therefore evaluate both predictive accuracy and uncertainty calibration to assess how closely the emulators match the offline Bayesian reference.

\paragraph{Contributions.} The main contributions of this work are:
\begin{enumerate}
    \item We formalize the \emph{MCMC Informed Neural Emulator} (MINE) paradigm, which decouples Bayesian posterior inference (via MCMC on a simulator) from deterministic function approximation (via neural networks trained on posterior-informed data). We provide a Wasserstein-based stability analysis showing that the deployment risk of any Lipschitz-bounded emulator is controlled by its training risk plus an explicit penalty proportional to the distribution relative to the deployment distribution, with constants depending only on Lipschitz moduli and second moments. This yields a principled justification that MCMC-informed training, training on (or near) the deployment/posterior law is bound-optimal, and it quantifies how finite-chain posterior approximations degrade performance.
    \item We present two complementary MINE realizations: a \emph{quantile (interval) emulator} that directly learns predictive quantiles to avoid sampling at inference time and a \emph{forward emulator} trained on posterior-informed input--output pairs that enables efficient posterior predictive draws without additional simulator calls. 
    \item For the forward emulator component, we instantiate the paradigm with a concrete, high-performing architecture: an AutoEncoder-based ODE neural network (AEODE) with time embeddings and attention.
\end{enumerate}

We emphasize, however, that MINE is \emph{architecture-agnostic} and can employ alternative neural ODE/neural operator backbones.
We evaluate the proposed framework on two representative examples: (a) a chemical kinetics ODE model with six chemical species and (b) the FaIR simple climate model, which simulates global temperature response to greenhouse gas (GHG) emissions. In addition, we provide experiments comparing AEODE against alternative ODE-style surrogate architectures to motivate our design choices.

The remainder of this paper is structured as follows: Section~\ref{sec:preliminaries} introduces preliminaries on Bayesian UQ with MCMC and presents two running examples. Section~\ref{sec:MCMCInformedTraining} introduces MCMC-informed training for neural emulators including a rigorous mathematical foundation for this framework, and Section~\ref{sec:methodology} details the two MINE realizations. Section~\ref{sec:experiments} describes experimental settings, and Section~\ref{sec:results} presents results on predictive fidelity, calibration, and efficiency. Section~\ref{sec:conclusion} concludes.

\section{Preliminaries}
\label{sec:preliminaries}
\subsection{Uncertainty quantification via Markov chain Monte Carlo}
Markov chain Monte Carlo sampling is a standard workhorse when it comes to quantifying and propagating both parametric and input uncertainties in mathematical models. We begin with a standard Bayesian inverse problem setup: Let $\boldsymbol{\theta} \in \Theta \subset \mathbb{R}^d$ denote the 
physical parameters of a model, $\boldsymbol{y} \in \mathbb{R}^n$ the observed data, and $F: \mathcal{X}\times\Theta \rightarrow \mathbb{R}^n$ the forward model mapping  the (controlled) input factors $\boldsymbol{x}\in\mathcal{X}$ given the parameters $\boldsymbol{\theta}\in\Theta$ to observables. We formulate the \emph{inverse parameter calibration problem} in a statistical setting, applying \emph{Bayes’ Theorem} to obtain a mathematical expression for the \emph{posterior probability density} of the model parameters. This posterior density can be explored numerically to receive the uncertainties. Let us formally  write our model as
$$
\boldsymbol{y}=F(\boldsymbol{x}\mid\boldsymbol{\theta})+\boldsymbol{\varepsilon},
$$
\vspace{-.1cm}
where 
$\boldsymbol{\varepsilon}$ presents the measurement noise. 

The posterior distribution of $\boldsymbol{\theta}$ - that is, those values of $\boldsymbol{\theta}$  that agree with data -  is given by the Bayes' formula
\begin{equation}\label{eqn:posterior}
p(\boldsymbol{\theta}\mid\boldsymbol{y})=\frac{p(\boldsymbol{y}\mid\boldsymbol{\theta})p(\boldsymbol{\theta})}{\int p(\boldsymbol{y}\mid\boldsymbol{\theta})p(\boldsymbol{\theta}) d\boldsymbol{\theta}},
\end{equation}
where $p(\boldsymbol{\theta}\mid\boldsymbol{y})$ denotes the posterior of 
$\boldsymbol{\theta}$ for given data $\boldsymbol{y}$, 
$p(\boldsymbol{y}\mid\boldsymbol{\theta})$ gives the likelihood of $\boldsymbol{y}$ for given $\boldsymbol{\theta}$, and 
$p(\boldsymbol{\theta})$ denotes the prior distribution of $\boldsymbol{\theta}$.
The Bayes' formula gives a framework to solve the inverse parameter calibration problem. The key benefit is that instead of a  single \emph{optimal} fit for the parameter $\boldsymbol{\theta}$ it provides the posterior distribution, \emph{all} parameter values that agree with given data. Consequently, any model prediction can be performed as an ensemble, using various parameters  $\boldsymbol{\theta}$ sampled from the posterior. In this way, the uncertainty of parameter values for model predictions is quantified. However, a direct evaluation of the Bayes' formula is only possible in some trivial cases. Therefore, we generate a representative set of samples $\{\boldsymbol{\theta}^{(i)}\}_{i=1}^N$ from the posterior using a Markov Chain Monte Carlo (MCMC) algorithm. MCMC methods provide a way to computationally implement the Bayes Theorem, i.e., to approximate the posterior distribution of the parameters, without explicitly evaluating the formula. Especially, 
the direct computation of the multidimensional integral is avoided. Instead of trying to sample directly from the unknown posterior distribution, a known \emph{proposal} distribution is used, and the samples from it are either accepted or rejected depending on the values of the current and proposed (non-normalized) likelihood values $p(\boldsymbol{y}\mid\boldsymbol{\theta})p(\boldsymbol{\theta})$.
Also, setting $p(\boldsymbol{\theta})$,  possible prior knowledge of the parameters may require expert knowledge. In our cases, we rely on either obvious physical constraints such as positivity or employ the assumptions used in literature. 
By iteratively repeating the MCMC steps -- propose, accept/reject -- a \emph{chain} of the values $\boldsymbol{\theta}$ is created, that can be proven to converge towards the true parameter posterior. However, the convergence may be impractically slow if the size and shape of the proposal distribution do not align with the underlying true distribution. This issue can be avoided by using adaptive methods that learn how to improve the proposal on the fly.  We use the  Delayed Rejection Adaptive Metropolis (DRAM), \cite{haarioDRAMEfficientAdaptive2006} that employs several multidimensional Gaussian proposal distributions, and updates the covariances of them based on the earlier sampled members in the chain. 
 
Once we have calibrated the model using input-output pairs, such as, e.g., historical inputs $\boldsymbol{x}_{\text{hist}}$ and outputs $\boldsymbol{y}_{\text{hist}}$ in our climate model example,to obtain an MCMC chain 
$\{\boldsymbol{\theta}^{(i)}\}_{i=1}^N$, we are interested in propagating the corresponding parametric uncertainty into the future. That is, for each sample $\boldsymbol{\theta}^{(i)}$ in the chain, we evaluate the forward model $F(\boldsymbol{x}\mid\boldsymbol{\theta}^{(i)})$ for potentially new data points $\boldsymbol{x}$, obtaining the corresponding outputs $\{F(\boldsymbol{x}\mid\boldsymbol{\theta}^{(i)})\}_{i=1}^N$. This corresponds to drawing from the posterior predictive distribution
\begin{equation}
    p(\boldsymbol{y}\mid\boldsymbol{y}_{\text{hist}}, \boldsymbol{x_{\text{hist}}};\boldsymbol{x})=\int p(\boldsymbol{y}\mid\boldsymbol{x},\boldsymbol{\theta})p(\boldsymbol{\theta}\mid\boldsymbol{y_{\text{hist}}}, \boldsymbol{x_{\text{hist}}})d\boldsymbol{\theta}.
    \label{eq:posterior_predictive_distribution}
\end{equation}
In some practical applications, one is interested in drawing from this posterior predictive distribution directly, in others, statistics such as relevant point estimates (e.g., means or medians) and measures of uncertainty (e.g., variances or credible intervals) suffice. Therefore, our framework offers two approaches to emulate: (1) the forward operator $F(\boldsymbol{x}\mid\boldsymbol{\theta})$ combined with (2) relevant quantiles of the posterior predictive distribution. Both emulators are fully deterministic for given $\boldsymbol{\theta}$, whose values can be randomly drawn from the sampled MCMC chain.

\subsection{Running Examples}
\label{subsec:running_examples}
We will use two examples to clarify the use of MCMC sampling as a tool within the MINE framework for physical models: a low-dimensional system from chemical kinetics, and a moderately high-dimensional simple climate model. While the former is a standard test example, the latter is a nice example for practically relevant high-dimensional settings while still being tractable on a standard computer. 

\subsubsection{Chemical Kinetics Model}

Consider the kinetic system between six chemical species $A,B,C,D,E,F$:
\begin{eqnarray}
 A + B \rightarrow C + F \\
 A + C \rightarrow D + F \\
 A + D \rightarrow E + F 
 \end{eqnarray}
where the reaction rates of the three reactions have to be estimated by the concentration data available of the component $A$. This example, called 'Himmel' in the following, is taken from \cite{ProcessAnalysisHimmelblau}, also available as a demo example of the MCMC library package in Matlab \footnote{https://github.com/mjlaine/mcmcstat}.

By denoting the three reaction constants as named $\theta_1,\theta_2,\theta_3$, we can write the respective ODE system as 

\begin{align}
\frac{dA}{dt} &= -\theta_1 AB - \theta_2 AC - \theta_3 AD \\
\frac{dB}{dt} &= -\theta_1 AB \\
\frac{dC}{dt} &= \theta_1 AB - \theta_2 AC \\
\frac{dD}{dt} &= \theta_2 AC - \theta_3 AD \\
\frac{dE}{dt} &= \theta_3 AD
\end{align}

\begin{figure}[pos=t]
	\centering
        \centerline{\includegraphics[width=.86\textwidth]{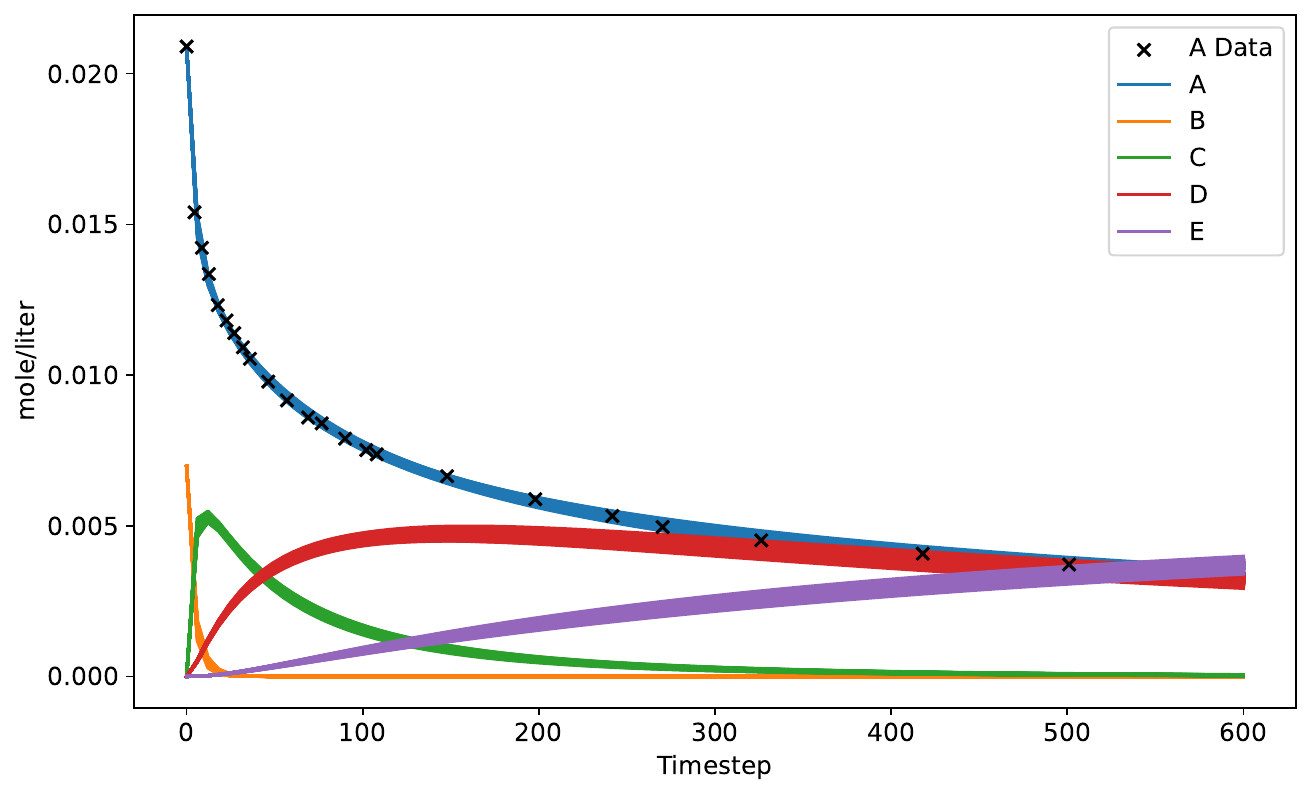}}
        	   	\captionsetup{font=small}
	\caption{\textbf{Himmel data fitting.}
		}
	\label{fig:himmelblau_fit}
\end{figure}

\begin{figure}[pos=t]
	\centering
		\centerline{\includegraphics[width=.86\textwidth]{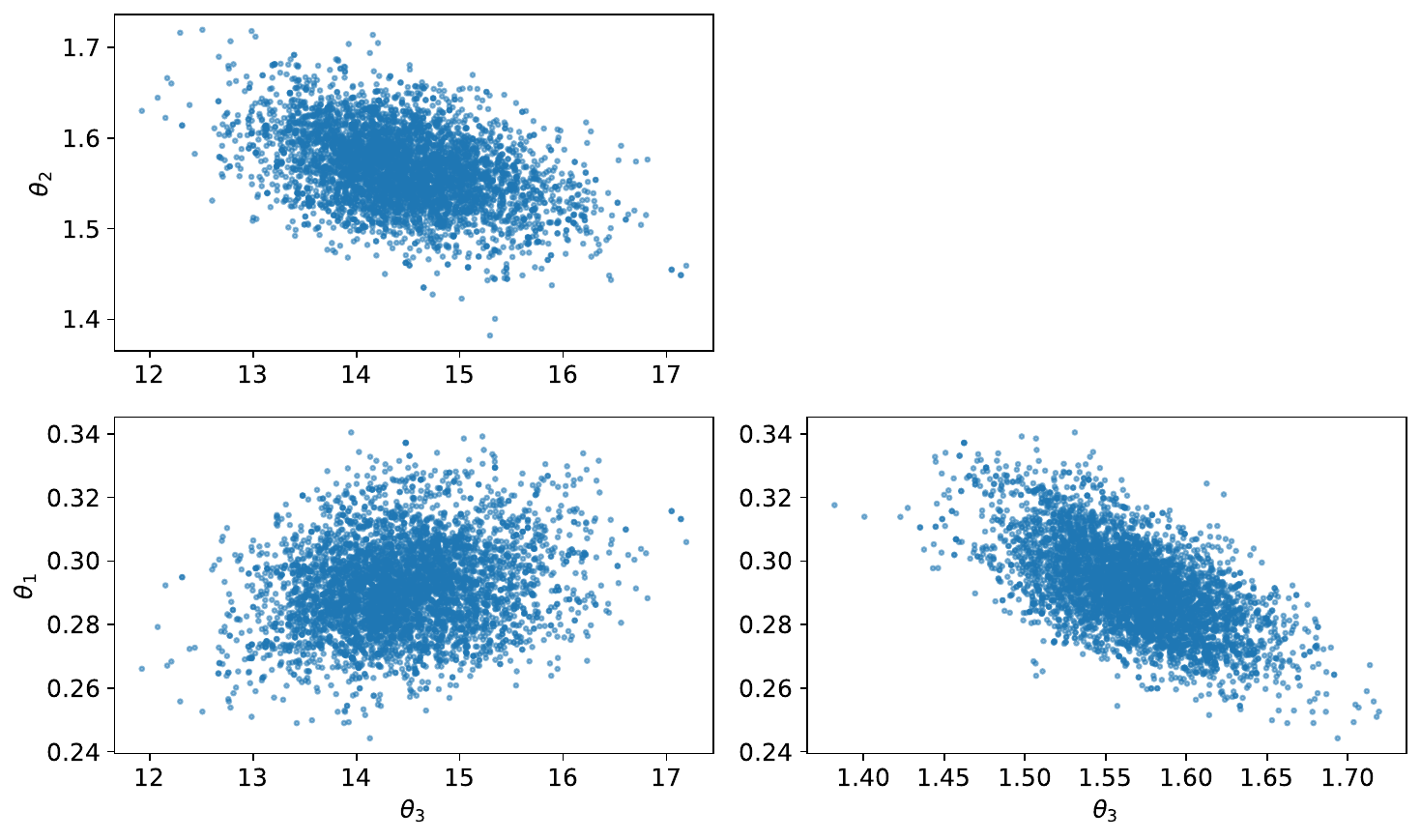}}
	   	\captionsetup{font=small}
            \caption{\textbf{2D scatter of Himmel parameters.}
		}
		\label{fig:himmel_posterior}
\end{figure}

Figure~\ref{fig:himmelblau_fit} shows how the model fits the data. In addition to the standard least squares fit it also shows the model solutions accepted by the MCMC sampling. The measurement noise is assumed to be i.i.d. Gaussian, with the noise variance estimated by the residuals of the fit. We see how the accepted model solutions agree with noise, i.e., they roughly just cover the measured concentration values.  Figure~\ref{fig:himmel_posterior} presents the pairwise 2D scatter plots of the parameters, as obtained by running an MCMC chain of length 5000 samples.

In the context of our Bayesian inverse problem setup $F(\boldsymbol{x}\mid\boldsymbol{\theta})$ is the forward operator corresponding to the ODE which returns the concentration of chemicals $A$ through $F$ in time $t$, namely $\boldsymbol{y}$. This forward operator $F(\boldsymbol{x}\mid\boldsymbol{\theta})$ has the input parameters $\boldsymbol{x}$ that contain the initial concentrations of the six chemical species $A$ through $F$, and  $\boldsymbol{\theta}$ which holds the reaction constants $\theta_1,\theta_2,\theta_3$.

\subsubsection{The FaIR Simple Climate Model}
\label{subsec:fair}

\begin{figure}[pos=t]
    \centering
    \includegraphics[width=1\textwidth]{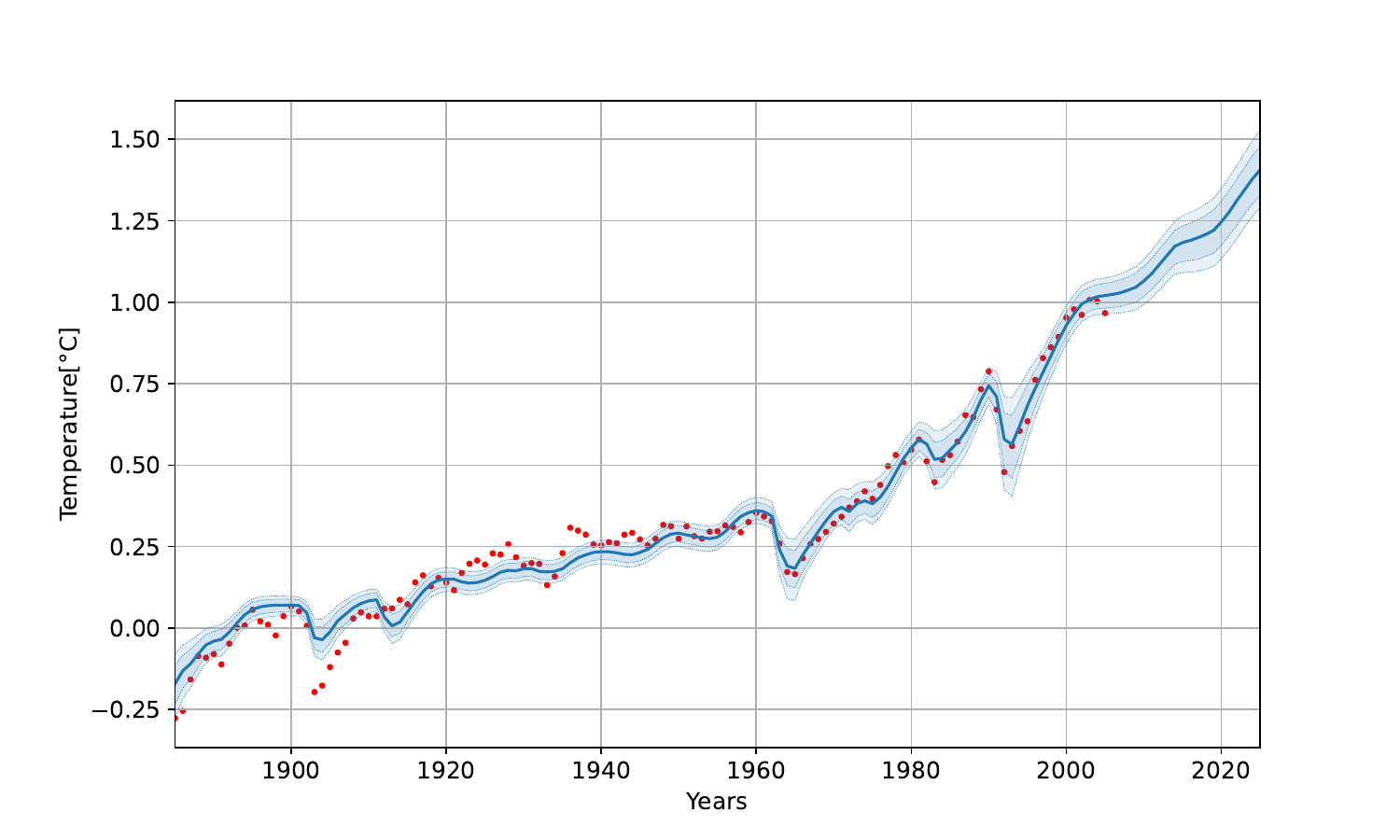}
    \caption{\textbf{FaIR calibration results (until 2005) and sampling from posterior predictive distribution (from 2005).} With historical temperature measurements in red, to the median, 90\% credible interval and 99\% credible interval in blue.}
    \label{fig:fair_data_fit}
\end{figure}

The Finite Amplitude Impulse Response (FaIR) 
model, see \cite{smith2018fair, gmd-14-3007-2021}, is a surrogate-type so-called \emph{simple} climate model, which aims to imitate the key features of more complicated earth system models. More precisely, the increase in temperature due to greenhouse gas emissions is modeled via a simplified carbon cycle in four main carbon reservoirs, namely geological, deep ocean, biosphere, and ocean mixed layer. We focus in this work on the FaIR version 1.6, cf. \cite{smith2018fair}, where the model dynamics are driven by annual greenhouse gas emissions as the primary input. For each year, the atmospheric concentrations of greenhouse gases are computed by accounting for their accumulation or decay, which results from exchanges with the aforementioned carbon cycle reservoirs, acting either as sinks or sources. These processes unfold over different timescales, depending on the specific reservoir involved and the prevailing temperature.

\noindent The governing equations for the amounts $R_i$ of a certain greenhouse gas (GHG) in each reservoir are given by
$$
\frac{\mathrm d R_i}{\mathrm{d} t}=a_iE_{\text{GHG}}(t)-\frac{R_i}{\alpha(T(t))\tau_i},\quad i=1,...,4,
$$
where $E_{\text{GHG}}(t)$ depicts the emission at time $t$, which are the core input factors to the model; $a_i$ contains the reservoir fractions ($\sum_{i=1}^4 a_i=1$); $\tau_i$ represents the lifespan of GHG in reservoir $i$; and $\alpha(\cdot)$ regulates the timescale adjustment through accounting for the temperature dependence of lifespan in each reservoir.

FaIR also calculates non-CO$_2$ greenhouse gas concentrations from emissions, aerosol forcing from aerosol precursor emissions, tropospheric and stratospheric ozone forcing from precursor emissions, and forcings from black carbon on snow, stratospheric methane oxidation to water vapor, contrails, and change in land use. Forcings from volcanic eruptions and solar irradiance fluctuations are supplied externally. These forcings are then converted to a temperature change. 

The model comprises a total of 20 physical parameters, each controlling specific aspects of the carbon cycle, radiative forcing, and climate response processes. In order to generate reliable climate projections, the FaIR model must be calibrated against historical data, including atmospheric greenhouse gas concentrations and global mean surface temperatures. The calibration process involves estimating the values of the model parameters such that the model output best fits the observed data over a historical time span. A key challenge in this process is the quantification of \emph{parameter uncertainty}. If uncertainties in parameter estimates are ignored or underestimated, the model may exhibit a systematic bias toward higher or lower temperature responses, commonly referred to as ``hot'' or ``cold'' biases, respectively. Such biases can distort projections and misinform policy decisions. To address this critical issue, a Bayesian parameter calibration framework has been proposed in \cite{weichel2024uncertaintyquantificationportfoliotemperature}. It incorporates prior knowledge and enables rigorous quantification of uncertainty in the parameter estimates. 

In the context of our Bayesian inverse problem setup $\boldsymbol{\theta} \in \mathbb{R}^{20}$ is the vector of the physical FaIR parameters to be estimated, $F(\boldsymbol{x}\mid\boldsymbol{\theta})$ is the forward model with the output $\boldsymbol{y} \in \mathbb{R}^n$ which denotes the vector of observed temperature anomalies at discrete time points $t_1, \ldots, t_n$ in the future, given the emission vector $\boldsymbol{x}\in \mathbb{R}^n$ at these time points. 

\paragraph{Parameterizing the emission pathway for scenario predictions}
In this work we are interested in quantifying both the parameter and input emission uncertainty. Therefore, we parametrize the emission input through two parameters, namely $E_0$, the absolute amount of GHG emissions in the base year, and $s$, the emission pathway for the following years. $s$ could be one of the SSP-RCP scenarios, which are a comprehensive set of plausible emission scenarios based on several socioeconomic narratives for the future \cite{o2014new}. In practical applications, we are interested in drawing from the corresponding posterior predictive distribution, cf. Equation~(\ref{eq:posterior_predictive_distribution}) given by
\begin{equation}
p(y\mid \boldsymbol{y}_{\text{hist}}, \boldsymbol{x}_{\text{hist}}; \boldsymbol{x}=g(s,E_0))=\int p(y\mid\boldsymbol{x}, \boldsymbol{\theta})\;p(\boldsymbol{\theta}\mid\boldsymbol{y}_{\text{hist}},\boldsymbol{x}_{\text{hist}})\;d\boldsymbol{\theta}
\end{equation}\label{eqn:post_prd_FAIR}
where $\boldsymbol{y}_{\text{hist}}$ are the historical temperature observations, $\boldsymbol{x}_{\text{hist}}$ the historical emissions, and $g(s,E_0)$ is a known function, which parameterizes the emission vector $\boldsymbol{x}$ into scenario $s$ and the base-year emissions $E_0$. As we will use separable Hilbert spaces later on, we embed the categorical variable $s$ into $\mathbb{R}^k$, e.g. via one-hot or scalar encoding.

Let us illustrate this setting, which is practically relevant, e.g., in financial climate risk management, see \cite{bourgey2024efficient, weichel2024uncertaintyquantificationportfoliotemperature}, using Figure~\ref{fig:fair_data_fit} below: To determine the posterior distribution of the FaIR parameters, $p(\boldsymbol{\theta}\mid\boldsymbol{y}_{hist},\boldsymbol{x}_{hist})$, a DRAM MCMC calibration as described in Section~\ref{sec:preliminaries} has been applied. The Figure~\ref{fig:fair_data_fit} shows the median, the 90\% credible interval, and the 99\% interval of the posterior predictive distribution and the data points the parameters were fitted to. The chains of the calibration can be seen in Figure~\ref{fig:fair_chains} in Appendix \ref{app:chains_fair}. For some use cases, one is interested in generating temperature pathways for given input parameters, whereas in other cases, one is merely interested in the spread of the distribution at a discrete point in the future, say, in the year 2100 in accordance with the Paris agreement.

\section{MCMC Informed Training}\label{sec:MCMCInformedTraining}

Our aim is to construct surrogate models that (i) emulate the underlying physical forward model and
(ii) reproduce the uncertainty summaries induced by the Bayesian posterior over parameters. In the
simple chemical kinetics example with fixed initial conditions shown in Figure~\ref{fig:himmelblau_fit}, the first task is
particularly transparent: we draw parameter samples $\{\boldsymbol{\theta}^{(s)}\}$ from the posterior, evaluate the
simulator to obtain concentration values $(A_t,\ldots,E_t)$ at time points of interest $t$, and fit a
regression model that maps the parameters to concentrations, i.e.,
$\boldsymbol{\theta} \mapsto (A_t,\ldots,E_t)$ (independently for each $t$ in Figure~\ref{fig:himmelblau_fit}). Since this example has only
three input parameters, even a second-degree polynomial provides an excellent fit (with $R^2 \approx
0.99$) across all positive concentrations. 
In addition to forward emulation, we also consider learning posterior-informed \emph{point estimates},
such as conditional means, as a function of the initial state. Here, the input is the vector of
initial concentrations $(A_0,\ldots,E_0)$, and the target is the concentration of a selected chemical
species at time $t$. Training pairs $\{(\boldsymbol{x}^{(i)},y^{(i)})\}_{i=1}^n$ are generated by sampling initial conditions
$\boldsymbol{x}^{(i)}$ independently from relevant intervals, sampling $\boldsymbol{\theta}$ independently from the posterior distribution, and
running the ODE system to obtain the corresponding output $y^{(i)}$ at time $t$. This yields a direct
mapping from initial concentrations to posterior-informed predictions at time $t$, which again can be
captured by simple regression using MSE loss in this low-dimensional example. Thus, in low-dimensional settings, a simple regressor can
already serve as an MCMC Informed forward surrogate, whereas more elaborate methods are required as the input dimension increases.

The FaIR model is our running example for such a moderately high-dimensional setting: When training an emulator of the forward model, taking $\boldsymbol{\theta}$ and $\boldsymbol{x}$ as inputs, we aim to reproduce how the model behaves under uncertainty in both input emissions and model parameters. As we have already discussed, these uncertainties come from two sources: input uncertainty, i.e., different initial values for the base-year emissions $E_0$ and plausible future emission pathways $s$. Together, these define the input distribution $p(\boldsymbol{x})$, describing which parts of the emission space we consider likely or relevant. The second source of uncertainty is parameter uncertainty, encoded into the distribution $p(\boldsymbol{\theta}\mid\boldsymbol{y}_{\text{hist}},\boldsymbol{x}_{\text{hist}})$ of model parameters and obtained from MCMC calibration against historical temperature and emission data.
The most favorable case is when the statistical conditions used to generate training data closely match those encountered at deployment.
When we have a credible target distribution for the inputs and parameters the emulator will face in practice---which we refer to as the deployment law $\nu_{\mathrm{dep}}$---it is natural to use this information to guide training-data generation.
In the FaIR setting, conditional on a fixed emissions scenario, the uncertainty in future model outputs induced by parameter uncertainty is naturally represented by the posterior predictive distribution.
Accordingly, drawing training samples from the calibrated posterior (or an approximation thereof) concentrates training on high-probability regions of the intended predictive setting and reduces reliance on extrapolation at prediction time.
In this sense, the MCMC posterior obtained from calibration provides a practical proxy for the deployment environment. In practice, however, we only approximate the true posterior using a finite MCMC chain, and the distribution of emissions scenarios at prediction time may be uncertain or application-dependent.
As a result, the training distribution may differ from the true deployment distribution (e.g., due to incomplete MCMC mixing or unbalanced scenario coverage), and performance can degrade under distribution shift. To quantify this effect, we provide a Wasserstein-based stability bound under standard regularity assumptions (Lipschitz continuity and finite second moments). 
The Lipschitz theory for out-of-distribution error bounds has been studied recently in \cite{guerra2025learning} and the theory provided in this section is inspired by this work. The proofs for all of the results stated in the remainder of this section can be found in the \ref{app:mathanalysis}.
Let $(\mathcal{X},\langle\cdot,\cdot\rangle_{\mathcal{X}}),(\Theta,\langle\cdot,\cdot\rangle_{\Theta}),(\mathcal{Y},\langle\cdot,\cdot\rangle_{\mathcal{Y}})$ be separable Hilbert spaces and set

$$
\mathcal{U}:=\mathcal{X}\times\Theta,
\qquad
\|(\boldsymbol{x},\boldsymbol{\theta})\|_{\mathcal{U}}^2=\|\boldsymbol{x}\|_{\mathcal{X}}^2+\|\boldsymbol{\theta}\|_{\Theta}^2.
$$
Assume the (true) forward model $F:\mathcal{U}\to \mathcal{Y}$ is Lipschitz,
$$
\|F(u)-F(v)\|_{\mathcal{Y}}\le L\,\|u-v\|_{\mathcal{U}}
\quad\text{for all }u,v\in\mathcal{U},
\qquad
\|F(0)\|_{\mathcal{Y}}<\infty,
$$
and that the hypothesis class $\mathcal{H}$ satisfies, for every emulator $\mathcal{E}\in\mathcal{H}$,
$$
\operatorname{Lip}(\mathcal{E})\le R,
\qquad
\|\mathcal{E}(0)\|_{\mathcal{Y}}\le B.
$$
For $\nu\in\mathcal{P}_2(\mathcal{U})$ (the set of probability measures on $\mathcal{U}$ with finite second moment),
define the squared-error risk
$$
\mathcal{R}_{\nu}(\mathcal{E}):=\mathbb{E}_{u\sim\nu}\,\|F(u)-\mathcal{E}(u)\|_{\mathcal{Y}}^2.
$$
Let $W_2$ denote the $2$–Wasserstein distance on $\mathcal{U}$ equipped with the metric $\|\cdot\|_{\mathcal{U}}$; see Villani~\cite{villani2009optimal} or Santambrogio~\cite{santabrogio2015optimal} for definitions and properties. Set
$$
C_1:=L+R,
\qquad
C_2:=\|F(0)\|_{\mathcal{Y}}+B,
$$
and for $\nu,\nu'\in\mathcal{P}_2(\mathcal{U})$ define
$$
c(\nu,\nu'):=C_1^2\,\sqrt{\,2\big(\mathbb{E}_{\nu}\|u\|_{\mathcal{U}}^2+\mathbb{E}_{\nu'}\|u\|_{\mathcal{U}}^2\big)}\;+\;2C_1C_2.
$$
Finally, for a prescribed deployment law $\nu_{\mathrm{dep}}\in\mathcal{P}_2(\mathcal{U})$ define the bound objective
$$
J(\nu):=\inf_{\mathcal{E}\in\mathcal{H}}\;\mathcal{R}_{\nu}(\mathcal{E})\;+\;c(\nu,\nu_{\mathrm{dep}})\,W_2(\nu,\nu_{\mathrm{dep}}).
$$

\begin{lemma}
\label{lem:shift}
For any $\nu,\nu'\in\mathcal{P}_2(\mathcal{U})$ and any $\mathcal{E}\in \mathcal{H}$,
\begin{equation}
\mathcal{R}_{\nu'}(\mathcal{E})\ \le\ \mathcal{R}_{\nu}(\mathcal{E})\;+\;c(\nu,\nu')\,W_2(\nu,\nu').
\end{equation}
\end{lemma}

By Lemma \ref{lem:shift}, a smaller Wasserstein distance implies a smaller upper bound on the increase in expected risk attributable to distribution mismatch. Note, however, that although the bound scales linearly with $W_2(\nu,\nu_{\mathrm{dep}})$, the prefactor $c(\nu,\nu_{\mathrm{dep}})$ depends on the second moments of $\nu$ and $\nu_{\mathrm{dep}}$, so the risk-shift guarantee is most informative when these moments are uniformly controlled (e.g., on a bounded operating domain). The following Proposition upper-bounds the best achievable deployment risk:

\begin{proposition}[Bound-optimality at the deployment law]
\label{prop:opt}
For any $\nu\in\mathcal{P}_2(\mathcal{U})$,
\begin{equation}
J(\nu)\ \ge\ J(\nu_{\mathrm{dep}})\ =\ \inf_{\mathcal{E}\in\mathcal{H}}\,\mathcal{R}_{\nu_{\mathrm{dep}}}(\mathcal{E}),
\end{equation}
with equality at $\nu=\nu_{\mathrm{dep}}$. Moreover,
\begin{equation}
0\ \le\ J(\nu)-J(\nu_{\mathrm{dep}})
\ \le\ 2\,c(\nu,\nu_{\mathrm{dep}})\,W_2(\nu,\nu_{\mathrm{dep}}).
\end{equation}
\end{proposition}

So, if we know or can approximate $\nu_{\mathrm{dep}}$, then training distribution matching is bound-optimal.
In the multi-scenario setting where the deployment environment selects scenario laws $\{\nu_k\}_{k=1}^K$ with probabilities $\{w_k\}_{k=1}^K$,
the resulting deployment law is the mixture $\nu_{\mathrm{mix}}=\sum_k w_k\nu_k$, and the same analysis motivates training on $\nu_{\mathrm{mix}}$: 

\begin{corollary}\label{thm:bo}
Over all $\nu\in\mathcal{P}_2(\mathcal{U})$,  $\nu_{\text{mix}}$ is a minimizer of $J(\nu)$, and
\begin{equation}
\inf_{\nu}J(\nu)=\inf_{\mathcal{E}\in\mathcal{H}}\mathcal{R}_{\nu_{\text{mix}}}(\mathcal{E}),\quad 0\leq J(\nu)-J(\nu_{\text{mix}})\leq 2 c(\nu,\nu_{\text{mix}})W_2(\nu,\nu_{\text{mix}}).
\end{equation}
\end{corollary}

Finally, when the posterior is approximated by an empirical MCMC measure $\widehat{\pi}_N$, we obtain an explicit bound on the degradation in $J$
in terms of $W_2(\widehat{\pi}_N,\pi)$, which vanishes as the empirical chain converges in Wasserstein distance under standard conditions: The posterior distribution 
$\pi(\boldsymbol{\theta}\mid\boldsymbol{y}_{\text{hist}},\boldsymbol{x}_{\text{hist}})$
cannot be represented exactly but is approximated by a finite empirical measure obtained from 
an MCMC simulation.
Consequently, the joint training distribution 
$\nu=\rho(\boldsymbol{x})\otimes \pi(\boldsymbol{\theta}\mid\boldsymbol{y}_{\text{hist}},\boldsymbol{x}_{\text{hist}})$
is replaced by its empirical counterpart
\begin{equation}\label{eqn:empirical}
\widehat{\nu}_N(\boldsymbol{x},\boldsymbol{\theta})
= \rho(\boldsymbol{x})\otimes\widehat{\pi}_N(\boldsymbol{\theta}),
\qquad
\widehat{\pi}_N(\boldsymbol{\theta})
=\frac{1}{N}\sum_{i=1}^{N}\delta_{\boldsymbol{\theta}^{(i)}},
\end{equation}
where $\{\boldsymbol{\theta}^{(i)}\}_{i=1}^{N}$ are the posterior samples generated by the chain. This substitution raises a natural question: 
\emph{how close is the finite-chain training law $\widehat{\nu}_N$ to the ideal deployment law 
$\nu_{\mathrm{dep}}$ and how much does this discrepancy affect the bound $J(\nu)$?}
To answer this, we next derive a quantitative extension of the previous results, 
showing that training on a finite MCMC chain remains 
asymptotically optimal, that is, with an approximation error that vanishes as the chain converges to the true posterior in Wasserstein distance. 
This analysis formalizes the intuition that the quality of MINE training depends primarily 
on how well the MCMC samples represent the underlying posterior distribution.

In the finite-chain analysis below, we take as deployment law the \emph{ideal} (infinite-chain) joint law
\begin{equation}\label{eq:ideal_dep_law}
\nu \;:=\; \rho\otimes\pi \in \mathcal P_2(\mathcal U),
\qquad\text{i.e.}\qquad \nu_{\mathrm{dep}}:=\nu,
\end{equation}
so that $J_{\nu}(\nu)=\inf_{\mathcal E\in\mathcal H}\mathcal R_{\nu}(\mathcal E)$ (since $W_2(\nu,\nu)=0$).

\begin{lemma}[Reduction of $W_2$ for product measures with a common marginal]
\label{lem:W2_product_reduction}
Let $\rho\in\mathcal P_2(\mathcal X)$ and $\pi,\widehat\pi\in\mathcal P_2(\Theta)$.
Equip $\mathcal U=\mathcal X\times\Theta$ with the product norm
$\|(\boldsymbol{x},\boldsymbol{\theta})\|_{\mathcal U}^2=\|\boldsymbol{x}\|_{\mathcal X}^2+\|\boldsymbol{\theta}\|_\Theta^2$.
Then
\begin{equation}\label{eq:W2_product_reduction}
W_2(\rho\otimes\pi,\ \rho\otimes\widehat\pi)
\;=\;
W_2(\pi,\widehat\pi).
\end{equation}
\end{lemma}

Now we can state our main result:
\begin{theorem}[Finite-chain approximation of the bound objective]
\label{thm:finite_chain}
Let $\pi\in\mathcal P_2(\Theta)$ and define the ideal joint law $\nu:=\rho\otimes\pi\in\mathcal P_2(\mathcal U)$.
Given posterior samples $\{\boldsymbol{\theta}^{(i)}\}_{i=1}^N$, define the empirical posterior
\begin{equation}
\widehat\pi_N:=\frac1N\sum_{i=1}^N\delta_{\boldsymbol{\theta}^{(i)}},
\qquad
\widehat\nu_N:=\rho\otimes\widehat\pi_N.
\end{equation}
Consider the objective with deployment law set to the ideal law, i.e.,\ $J_{\nu}(\cdot)$  with $\nu$ given by \eqref{eq:ideal_dep_law}. 
Then
\begin{equation}\label{eq:finite_chain_bound}
0\ \le\ J_{\nu}(\widehat\nu_N)-J_{\nu}(\nu)
\ \le\ 2\,c(\widehat\nu_N,\nu)\,W_2(\widehat\pi_N,\pi).
\end{equation}
Moreover, assume $\{\boldsymbol{\theta}^{(i)}\}_{i\ge 1}$ is a (Harris) ergodic Markov chain with invariant distribution $\pi$ such that
(i) the ergodic theorem applies to every bounded continuous $\varphi:\Theta\to\mathbb R$, and
(ii) the ergodic theorem applies to the (unbounded) function $\varphi(\boldsymbol{\theta})=\|\boldsymbol{\theta}\|_{\Theta}^{2}$, i.e.
\begin{equation}
\frac{1}{N}\sum_{i=1}^{N}\|\boldsymbol{\theta}^{(i)}\|_{\Theta}^{2}\;\xrightarrow[N\to\infty]{a.s.}\;\int_{\Theta}\|\boldsymbol{\theta}\|_{\Theta}^{2}\,\pi(d\boldsymbol{\theta}).
\label{eq:slln-second-moment}
\end{equation}
Then, almost surely, $W_2(\widehat\pi_N,\pi)\to 0$ and $c(\widehat\nu_N,\nu)\to c(\nu,\nu)$. Consequently,
$J_{\nu}(\widehat\nu_N)\to J_{\nu}(\nu)$ almost surely.
\end{theorem}

\begin{remark}
If one keeps a fixed external deployment law $\nu_{\mathrm{dep}}\neq \nu$,
then controlling $J_{\nu_{\mathrm{dep}}}(\widehat\nu_N)-J_{\nu_{\mathrm{dep}}}(\nu)$ requires, in addition to
$W_2(\widehat\nu_N,\nu)$, a control of the penalty mismatch
$W_2(\widehat\nu_N,\nu_{\mathrm{dep}})-W_2(\nu,\nu_{\mathrm{dep}})$ (e.g.\ via the triangle inequality)
and any variability of $c(\cdot,\nu_{\mathrm{dep}})$ with respect to its first argument.
\end{remark}

\newpage
\section{Neural Emulators}
\label{sec:methodology}
The proposed \emph{MCMC Informed Neural Emulator} (MINE) framework enables scalable and task-specific uncertainty quantification without embedding stochasticity directly into the network architecture by decomposing the task of Bayesian uncertainty quantification into two modular steps: \textbf{Step 1} consists of sampling from the posterior distribution of physical parameters using MCMC, and \textbf{Step 2}, training a deterministic neural network emulator using posterior informed training data as described in Section~\ref{sec:MCMCInformedTraining}. 

To explore the posterior predictive distribution given the dataset generated in Step~1, we present two emulators tailored for different use cases: 
\begin{enumerate}
    \item \textbf{A quantile (interval) emulator:} We directly train the neural network $\mathcal{N}_{\boldsymbol{\omega'}}: \mathcal{X}\rightarrow \mathbb{R}^m$ with learnable parameters $\boldsymbol{\omega'}$ to learn the summary statistics of the posterior predictive distribution, specifically the quantile interval.
    \item \textbf{A forward emulator:} We train a surrogate neural network $\mathcal{F}_{\boldsymbol{\omega}}: \mathcal{X}\times\Theta \rightarrow \mathbb{R}^n$ with learnable parameters $\boldsymbol{\omega}$ using posterior samples $\boldsymbol{\theta}$ to emulate the forward operator $F(\boldsymbol{x}\mid\boldsymbol{\theta})$ in the sampling. This enables a fast approximate sampling of the entire posterior predictive distribution. Importantly, the MINE framework does not require a specific architecture here: $\mathcal{F}_{\boldsymbol{\omega}}$ can be implemented by any suitable neural ODE / neural operator surrogate.

\end{enumerate}
In both cases, the neural network is trained in a supervised manner and can be tailored to the specific task. 
In the following subsections, we elaborate on the particular architectures for $\mathcal{F}_{\boldsymbol{\omega}}$ and $\mathcal{N}_{\boldsymbol{\omega'}}$ used in this work. 

\subsection{Quantile Emulator}

Direct computation of the posterior predictive distribution using Monte Carlo simulations involves repeated evaluations of the forward model, which can be prohibitively expensive when applied to large ensembles or in time-sensitive decision-making contexts - even when using a neural emulator for the forward model. While our forward ODE emulator can be used to accelerate the forward model evaluations within the Monte Carlo sampling, the second, complementary, part of the MINE framework, the quantile emulator, can be applied when sampling full paths obtained from the full posterior predictive distribution is not required; instead, the focus lies on quantiles of this distribution as measures of uncertainty. We utilize the neural network $\mathcal{N}_{\boldsymbol{\omega'}}: \mathcal{X}\rightarrow \mathbb{R}^m$ with learnable parameters $\boldsymbol{\omega'}$ to generate posterior interval estimates without the computational burden of traditional sampling methods for each new data instance. Such an approach is particularly advantageous in applications like financial climate risk management, see, e.g., \cite{gobetMetamodellingPathsSimple2025, weichel2024uncertaintyquantificationportfoliotemperature}, where timely estimates are crucial. 

Let us first formalize the problem setup and key concepts: Throughout, we assume a data model in which we have random examples $\boldsymbol{Z} = (\boldsymbol{X}, Y)$ that were created using our forward model $F(\boldsymbol{x}\mid\boldsymbol{\theta})$. It consists of features $\boldsymbol{X} \in \mathcal{X}$ and a response $Y \in \mathcal{Y} \subseteq \mathbb{R}$. The features are chosen with respect to the considered model, and the respective response is then created using the forward model and a random sample of the chain from the MCMC calibration, i.e., $Y= F(\boldsymbol{x}=\boldsymbol{X}\mid\boldsymbol{\tilde{\theta}})$ and $ \boldsymbol{\tilde{\theta}}\sim p(\boldsymbol{\theta}\mid\boldsymbol{y})$. We are given a finite training set $\mathcal{D}=\{\boldsymbol{Z}^{(1)}, \boldsymbol{Z}^{(2)}, \dots, \boldsymbol{Z}^{(n)}\}$ and will make predictions about a new test example $\boldsymbol{Z}^{(n+1)}$. 

We predict $\hat q_{0.05}(\boldsymbol{x})$ and $\hat q_{0.95}(\boldsymbol{x})$ with a ReLU feed‑forward net (two hidden layers, 20 units each) using a pinball-loss objective with a non-crossing regularization term
\begin{equation}
\begin{split}    
\mathcal L =& \sum_{i} \bigl[L_{0.05}(\hat q_{0.05}(\boldsymbol{X}^{(i)}),Y^{(i)})+L_{0.95}(\hat q_{0.95}(\boldsymbol{X}^{(i)}),Y^{(i)})\bigr] 
\\& + \lambda\sum_{i}\max\bigl\{0,\hat q_{0.05}(\boldsymbol{X}^{(i)})-\hat q_{0.95}(\boldsymbol{X}^{(i)})\bigr\},
 \end{split}
 \label{eq:multitarget}
\end{equation}
where $L_{\tau}$ is the so-called pinball loss:
\begin{equation}
L_{\tau}\!\bigl(\hat q_{\tau}(\boldsymbol{x}),\,y\bigr)
=
\begin{cases}
\tau\bigl(y-\hat q_{\tau}(\boldsymbol{x})\bigr),      & \text{if } y\ge \hat q_{\tau}(\boldsymbol{x}),\\[6pt]
(\tau-1)\bigl(y-\hat q_{\tau}(\boldsymbol{x})\bigr),  & \text{if } y<\hat q_{\tau}(\boldsymbol{x}).
\end{cases}
\label{eq:pinball-piecewise}
\end{equation}

\noindent Let $\tau_1,\tau_2,\dots,\tau_m \in (0,1)$ be a set of quantile levels of interest (e.g., 0.05 and 0.95 for a 90\% prediction interval). The corresponding quantile neural network $\mathcal{N}_{\boldsymbol{\omega'}}(\boldsymbol{x})$ is parametrized by weights $\boldsymbol{\omega'}$, which for given features $\boldsymbol{x}$ outputs $m$ values $(\hat q_{\tau_1}(\boldsymbol{x}), \ldots, \hat q_{\tau_m}(\boldsymbol{x}))$, which are estimates of the conditional quantiles: Let $\pi(d\boldsymbol{\theta})$ denote the calibrated parameter posterior, e.g.
$\pi(d\boldsymbol{\theta})=p(\boldsymbol{\theta}\mid \boldsymbol{y}_{\mathrm{hist}},\boldsymbol{x}_{\mathrm{hist}})\,d\boldsymbol{\theta}$.
For a fixed input $\boldsymbol{x}\in\mathcal X$, define the posterior predictive conditional CDF
\begin{equation}
F_{Y\mid \boldsymbol{X}}(y\mid \boldsymbol{x})
:=\mathbb P(Y\le y\mid \boldsymbol{X}=\boldsymbol{x})
=\int_{\Theta}\mathbb P(Y\le y\mid \boldsymbol{X}=\boldsymbol{x},\boldsymbol{\theta})\,\pi(d\boldsymbol{\theta}).
\label{eq:postpred-cdf}
\end{equation}
For $\tau\in(0,1)$, the conditional $\tau$-quantile of $Y$ given $\boldsymbol{X}=\boldsymbol{x}$ is defined by
\begin{equation}
q_{\tau}(\boldsymbol{x})\;:=\;\inf\{y\in\mathbb R:\;F_{Y\mid \boldsymbol{X}}(y\mid \boldsymbol{x})\ge \tau\}.
\label{eq:quantile-def}
\end{equation}
Let us emphasize again that the key idea here is that the training data is generated from the MCMC Informed training distribution obtained in Step 1 of the framework.

\subsection{ODE Forward Emulator: AEODE Neural Network}
\label{subsec:AEODE}

Classical neural ODE approaches learn a parameterized vector field and generate trajectories through sequential time integration. In contrast, the proposed AEODE adopts a neural operator perspective and directly maps model parameters and initial conditions to full solution trajectories on a fixed time grid in a single forward evaluation.

Inspired by recent advances in deep learning for scientific modeling, we propose a deep neural network capable of learning regression mappings from MCMC-sampled data. Unlike existing neural ODE emulators~\cite{neuralode,ode_1,ode_2,ode_3}, our method focuses on parameterized ODEs, aiming to learn the correlations between ODE trajectories and their hidden parameters (e.g., rate coefficients in chemical reactions). Instead of relying on explicit ODE solvers as in~\cite{neuralode}, we design a time-dependent neural network that jointly learns the dependencies between time steps and hidden parameters under fixed initial conditions. Our main novelties are as follows:
\begin{enumerate}
\item Neural operator design: rather than modeling time evolution step by step as in Neural ODEs (e.g., Torchdiffeq~\cite{neuralode} or ChemiODE~\cite{chemiode}), we integrate an attention mechanism~\cite{attn1} to model temporal correlations in the latent domain. Furthermore, different from other neural operator methods~\cite{attn_no_1,attn_no_2} purely utilizing attention mechanisms for operator learning, we propose a learnable time embedding that provides a continuous feature representation that enhances the network’s ability to model ODE dynamics.
\item Physics-informed supervision: we incorporate derivative-based losses on both first- and second-order temporal gradients and total mass conservation to enforce physical consistency during training.
\end{enumerate}

The resulting AutoEncoder-based ODE network (AEODE) can flexibly estimate different trajectory paths corresponding to various ODE parameters $\boldsymbol{\theta}$. Thanks to its end-to-end differentiable architecture optimized for GPUs, AEODE achieves an approximate 10× speedup compared to conventional numerical solvers. The additional physical losses ensure physically consistent and accurate ODE trajectory generation.

\begin{figure*}[pos=b]
	\centering
		\centerline{\includegraphics[width=\textwidth]{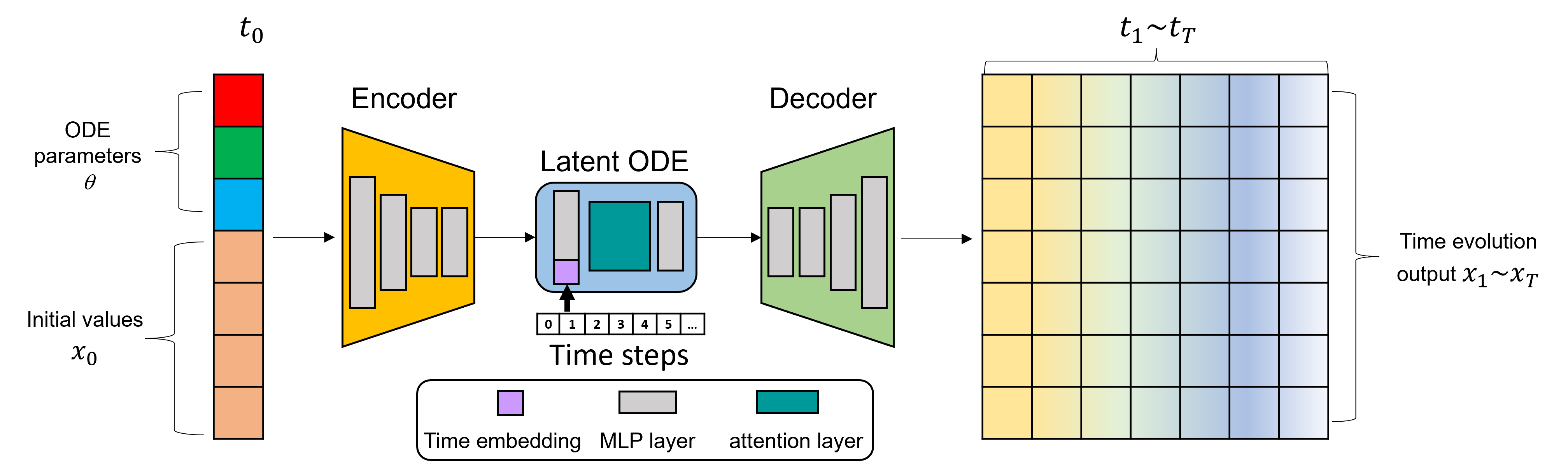}}
		\caption{\small{\textbf{The structure of the proposed \net.} 
		}}
		\label{fig:network}
\end{figure*}

\paragraph{Overall structure}
The overall structure is shown in Figure~\ref{fig:network}. Let us denote the initial ODE value as $\boldsymbol{x}_0\in \mathbb{R}^N$, where $N$ is the number of input elements, and the ODE parameters as $\boldsymbol{\theta}\in\mathbb{R}^d$, where $d$ is the number of model parameters. Mathematically, given the time steps $t_i=0,1,...,T$, the proposed \net learns a neural ODE function that predicts time-dependent output values $\boldsymbol{x}_t\in\mathbb{R}^N$ based on different ODE parameters $\boldsymbol{\theta}$ as,

\begin{equation}
\boldsymbol{x}(t,\boldsymbol{\theta})
=\psi\!\left(\varphi(\boldsymbol{x}_0,\boldsymbol{\theta})+\int_{t_0}^{t} f\!\left(z(s,\boldsymbol{\theta}),\,s;\,\boldsymbol{\theta}\right)\,ds\right),
\qquad t\in[t_0,t_T],
\label{eq:ode}
\end{equation}
where $\varphi(\cdot,\cdot)$ is the encoder producing the latent initial condition, $f$ is the latent vector field, and $\psi$ is the decoder mapping the latent trajectory back to the observed state space. The dynamics of the ODE process are characterized as $\boldsymbol{x}(t,\boldsymbol{\theta})$. Given initial values $\boldsymbol{x}_0$, we have the initial state of the latent space $\boldsymbol{z}_0=\varphi(\boldsymbol{x}_0,\boldsymbol{\theta})$, which serves as the initial condition for the latent ODE process. The encoder and decoder are made of multiple layers of MLP for fully connected computation. In order to improve the long-term predictions. We propose two techniques for the latent ODE operation: 1) time embedding, and 2) attention for more accurate space-time evolution.

\paragraph{Time embedding}
We observe that some physical phenomena oscillate over time, like chemical reactions. This behavior can be injected into the neural network as prior knowledge that predetermines the tendency of time evolution. 

Based on this observation, we propose to utilize frequency time encoding. Mathematically, given the time steps $t_i=0,1,...,T$, we project them onto a higher dimensional space $\mathbb{R}^{2L}$ as follows,

\begin{equation}
\lambda(t)
=\Big(\sin(2\pi \omega_{1} t),\cos(2\pi \omega_{1} t),\ldots,
      \sin(2\pi \omega_{L} t),\cos(2\pi \omega_{L} t)\Big)\in\mathbb R^{2L},
\label{eq:time-embed}
\end{equation}
where $\omega=(\omega_1,\ldots,\omega_L)\in\mathbb R^{L}$ are learnable frequency parameters (initialized, e.g., on a log-spaced grid). We represent the time as a combination of sine and cosine operators so that the network can learn to adjust the reaction frequency.

\paragraph{Attention for nonlocal feature extraction}
Attention~\cite{attn1} is commonly used in natural language~\cite{attn2,attn3} and image processing~\cite{attn4} because of its high efficiency in learning nonlocal data correlations. Similarly, we could use the same technology for ODE modeling. The motivation comes from the aforementioned time encoding, that we can project the time code to a higher dimension. Meanwhile, the latent values form a time-conditioned latent representation by $\boldsymbol{z}_0=\phi(\boldsymbol{x}_0,\boldsymbol{\theta})$ and $\boldsymbol{z}_\lambda(t)=\boldsymbol{z}_0+\lambda(t)$.
In other words, given the value at the initial time $t_0$, we construct a time-aware representation by embedding multiple future time steps, e.g., $t_1.t_2,\dots$, using the aforementioned time embedding. Each of these time steps is mapped to a high-dimensional vector using sine and cosine functions to preserve temporal order and periodicity. We then concatenate the initial value with each time embedding to create a set of 2D latent vectors that encode both the ODE element states and temporal context. Therefore, we 
can implement the time-dependent attention module for nonlocal computation as,

\begin{equation}
\boldsymbol{z}_{\lambda}^{\mathrm{out}}(t)
= \boldsymbol{z}_{\lambda}^{\mathrm{in}}(t)
+\sigma\!\left(\frac{Q_z K_z^\top}{\sqrt d}\right)V_{\boldsymbol{z}},
\qquad
Q_{\boldsymbol{z}}=\boldsymbol{z}_{\lambda}^{\mathrm{in}}(t)W_Q,\;K_{\boldsymbol{z}}=\boldsymbol{z}_{\lambda}^{\mathrm{in}}(t)W_K,\;V_{\boldsymbol{z}}=\boldsymbol{z}_{\lambda}^{\mathrm{in}}(t)W_V.
\label{eq:attn}
\end{equation}

\noindent The $W_*$ is the learnable parameters, $d$ is the dimension of the feature maps. $\sigma$ is the softmax function to normalize the auto-correlation matrix.

\paragraph{Losses for optimization}
To train the whole \net, not only do we utilize the commonly used Mean Squared Errors (MSE) between prediction and ground truth, but we also propose to utilize the first- and second-order derivatives, and total mass conservation loss. Mathematically, we can define the overall losses as

\begin{equation}
Loss=\alpha_1 L_{recon}+\alpha_2 L_{d1}+\alpha_3 L_{d2} +\alpha_4 L_{idn}+\alpha_5 L_{mass}
\label{eq:loss}
\end{equation}

\noindent In equation~(\ref{eq:loss}), $\alpha_1$ to $\alpha_5$ are the weighting parameters that balance all five loss terms. For $L_{recon}$ is the reconstruction loss that measures the discrepancy between the prediction and ground truth as $L_{recon}=\|\boldsymbol{x}(t)-\boldsymbol{x'}(t)\|^2$. The first-order gradient loss enforces the predicted trajectories to closely follow the ground truth over time as $L_{d1}=\|\frac{d\boldsymbol{x}(t)}{dt}-\frac{d\boldsymbol{x'}(t)}{dt}\|^2$. Similarly, we can define the second-order gradient loss as $L_{d2}=\|\frac{d\boldsymbol{x}^2(t)}{dt^2}-\frac{d\boldsymbol{x'}^2(t)}{dt^2}\|^2$. Meanwhile, we can also enforce that the \net should preserve the initial condition unchanged during the training process. Hence we can define the identity loss as $L_{idn}=\|\boldsymbol{x}_0-\psi(\varphi(\boldsymbol{x}_0,\boldsymbol{\theta})\|^2$. Finally, we define the total mass conservation loss, ensuring that the mass of the predicted trajectory aligns with the mass of the ground truth at each time step. We have $L_{mass}=\|\mathcal{M}(\boldsymbol{x}(t))-\mathcal{M}(\boldsymbol{x'}(t))\|^2$, where $\mathcal{M}$ is the summation of all ODE elements.

Let us interpret the design of the proposed AEODE in detail. Note that the latent integral in Eq.~(\ref{eq:ode}) equals the latent state, 
\begin{equation}
\boldsymbol{z}(t,\boldsymbol{\theta})=\varphi(\boldsymbol{x}_0,\boldsymbol{\theta})+\int_{t_0}^{t} f(\boldsymbol{z}(s,\boldsymbol{\theta}),s; \boldsymbol{\theta})\,\mathrm{d}s.
\end{equation}
Although it may not hold exactly here, consider for a moment the latent flow expressed in Duhamel/variation-of-constants form:
\begin{equation}\label{eqn:Duhmael}
\boldsymbol{z}(t,\boldsymbol{\theta})=\Phi_{\boldsymbol{\theta}}(t,t_0)\,\varphi(\boldsymbol{x}_0,\boldsymbol{\theta})
\;+\;\int_{t_0}^{t}\Phi_{\boldsymbol{\theta}}(t,s)\,u_{\boldsymbol{\theta}}(s)\,\mathrm{d}s,
\end{equation}
which defines the latent propagator (operator-valued kernel)
\(\mathcal{K}_{\boldsymbol{\theta}}(t,s):=\Phi_{\boldsymbol{\theta}}(t,s)\). On a discrete time grid the integral term is approximated by the quadrature

$$\int_{t_0}^{t_i}\mathcal{K}_{\boldsymbol{\theta}}(t_i,s)\,u_{\boldsymbol{\theta}}(s)\,\mathrm{d}s
\approx \sum_j \Delta s_j\,\mathcal{K}_{\boldsymbol{\theta}}(t_i,s_j)\,u_{\boldsymbol{\theta}}(s_j).
$$
AEODE's time-aware attention implements this quadrature: by forming queries/keys from the time-conditioned latent \(\big[\varphi(\boldsymbol{x}_0,\boldsymbol{\theta});\lambda(t)\big]\) and values \(V_j\) from the local drivers, the learned attention weights \(w_{ij}\) act as parameter-conditioned quadrature weights that approximate \(\mathcal{K}_{\boldsymbol{\theta}}(t_i,s_j)\) (up to a calibrating readout). Consequently the decoder receives the encoder output \(\varphi(\boldsymbol{x}_0,\boldsymbol{\theta})\) plus an attention-produced integral increment \(\sum_j w_{ij}V_j\), yielding a consistent operator-level interpretation of Eq.~(\ref{eq:ode}) in terms of \(\mathcal{K}_{\boldsymbol{\theta}}\).

\begin{remark}
Equation~\eqref{eqn:Duhmael} should be read with care. The representation is the classical variation-of-constants (Duhamel) formula and holds \emph{exactly} for linear (possibly time-varying) latent dynamics
$\dot{\boldsymbol{z}}(t)=A_{\boldsymbol{\theta}}(t)\,\boldsymbol{z}(t)+\boldsymbol{u}_{\boldsymbol{\theta}}(t)$, where $\Phi_{\boldsymbol{\theta}}(t,s)$ is the state-transition operator generated by $A_{\boldsymbol{\theta}}(\cdot)$.
In our architecture the latent dynamics are in general nonlinear, $\dot{\boldsymbol{z}}(t)=f_{\boldsymbol{\theta}}(\boldsymbol{z}(t),t)$, for which no global linear propagator
$\Phi_{\boldsymbol{\theta}}(t,s)$ exists in general. Accordingly, we interpret~\eqref{eqn:Duhmael} as an \emph{operator-inspired analogy}:
the attention mechanism produces weights $K_{\boldsymbol{\theta}}(t,s)$ that act as a data-driven influence kernel describing how information at time $s$
contributes to the latent state at time $t$. More precisely, along a given trajectory $\boldsymbol{z}(\cdot)$ we may linearize the dynamics,
$$
\delta\dot{\boldsymbol{z}}(t)=J_{\boldsymbol{\theta}}(t)\,\delta \boldsymbol{z}(t)+\delta \boldsymbol{u}(t), \qquad
J_\theta(t):=\nabla_{\boldsymbol{z}} f_{\boldsymbol{\theta}}(\boldsymbol{z}(t),t),
$$
so that the corresponding perturbations satisfy a Duhamel formula with a trajectory-dependent transition operator
$\Phi_{\boldsymbol{\theta}}^{(\boldsymbol{z})}(t,s)$ generated by $J_{\boldsymbol{\theta}}(\cdot)$. In this sense, $K_{\boldsymbol{\theta}}(t,s)$ can be viewed as a learned approximation to
$\Phi_{\boldsymbol{\theta}}^{(z)}(t,s)$ (or to a truncated Volterra-type expansion) \emph{along the data manifold}, rather than as an exact propagator
for the full nonlinear system.
\end{remark}

\section{Experiments}
\label{sec:experiments}

This section presents the experimental setup of the proposed methodology to evaluate the effectiveness of the MINE paradigm. We distinguish two experimental goals: (i) evaluating the end-to-end MINE framework for uncertainty-aware prediction (forward-emulator sampling and quantile/interval prediction), and (ii) motivating our concrete choice of forward-emulator backbone.
While the MINE framework is architecture-agnostic and the forward emulator can in principle be any neural ODE / neural operator surrogate, in this paper we instantiate it with AEODE due to its accuracy and efficiency on the studied dynamical systems. Accordingly, some experiments focus specifically on AEODE design choices and comparisons, while the remaining experiments focus on the end-to-end UQ performance enabled by the MINE pipeline.
We apply the ODE forward emulator to both the FaIR climate model and the chemical kinetics model, and the quantile emulator to the FaIR climate model only, since the posterior predictive distribution of the simple chemical kinetics is very narrow and does not require a neural network for emulation.

\subsection{Evaluation Metrics}

\paragraph{Quantile Emulator}

To assess the performance of the quantile emulator, we first compute the pinball loss, which is also utilized during the model training phase. Additionally, we perform tests using empirical quantile data. Specifically, for given inputs $\boldsymbol{x}$, we calculate the conditional quantile values $q_{\tau_j}(\boldsymbol{x})$ by numerically approximating the posterior predictive distribution, represented by Monte Carlo draws $\{\widehat y^{(m)}\}_{m=1}^{M}$ from the posterior predictive law
(obtained by sampling $E_0\sim p(E_0\mid\eta)$ and $\boldsymbol{\theta}\sim p(\boldsymbol{\theta}\mid \boldsymbol{y}_{\mathrm{hist}},\boldsymbol{x}_{\mathrm{hist}})$ and evaluating the simulator). Next, we evaluate the mean squared error (MSE) between the empirical quantiles and those generated by the emulator. We also compare the average size of the credible intervals for both the empirical data and the predictions made by the neural network. Finally, we determine the mean 90\%-coverage, which is the average proportion of test sample elements in $\{\widehat y^{(m)}\}_{m=1}^{M}$ that fall within the 90\% credible interval predicted by the emulator.

\paragraph{Forward Emulator}

To evaluate the ODE forward emulator, we compare emulator outputs against the corresponding physical-model outputs using mean squared error (MSE), mean absolute error (MAE), root mean squared error (RMSE), and mean bias error (MBE). Furthermore, we report wall-clock runtime to quantify computational efficiency relative to direct numerical simulation and to assess feasibility for large-scale use.

\subsection{Emulating FaIR with the Quantile Emulator}

With the MINE Posterior Interval Predictor, we estimate the measures of uncertainty in the posterior predictive distribution. Specifically for the FaIR model we want to show how to model additional uncertainties in the input factors $\boldsymbol{x}$. This is motivated by the real-world need to account for both the parameter uncertainty and the uncertainty in future emissions. Direct evaluation of the corresponding \emph{posterior predictive distribution} requires high-dimensional integration over both the model parameter space and future emission trajectories. Our neural network emulator is well suited for this approach, since it scales with this high-dimensional integration. 

Parameter uncertainty is reflected in the posterior distribution as described in Equation~\ref{eqn:posterior}. The uncertainties in the input factors $\boldsymbol{x}$ are modeled by treating the base year emissions $E_0$ as a random variable, reflecting the fact that estimates of global emissions are inherently uncertain and may even be significantly biased, for example due to unreported methane leakages. Therefore, we introduce a set of probability distributions parameterized by distribution hyperparameters $\boldsymbol{\eta} \in \mathcal{E}$. The hyperparameter $\boldsymbol{\eta}_1$ serves as a binary flag indicating the choice of distribution: $\boldsymbol{\eta}_1 = 0$ corresponds to a normal distribution, while $\boldsymbol{\eta}_1 = 1$ specifies a lognormal distribution. The lognormal case is used to model situations in which global emissions are systematically underestimated. The hyperparameter $\boldsymbol{\eta}_2$ defines the location parameter of the distribution, representing a horizontal shift along the x-axis. This is the expected value for global emissions $E_{global}$. The hyperparameter $\boldsymbol{\eta}_3$ specifies the distribution's scale and allows the user to express uncertainty in the global emission. For the normal distribution, this corresponds to the standard deviation. In the lognormal case, $\boldsymbol{\eta}_3$ controls the spread of the distribution via the standard deviation of the underlying normal distribution. Figure~\ref{pdfs_eta} illustrates three example distributions corresponding to different hyperparameter vectors $\boldsymbol{\eta}$, while Figure~\ref{fig:eta_paths} shows the emission pathways for three distinct SSP-RCP scenarios.

\begin{figure}[pos=b]
    \centering
    \subfloat[\textbf{Example distributions of $E_0$} with normal distribution and small sd. (blue), normal distribution and large sd. (orange), and lognormal distribution and small sd. (red).]
    {
        \includegraphics[width=0.45\textwidth]{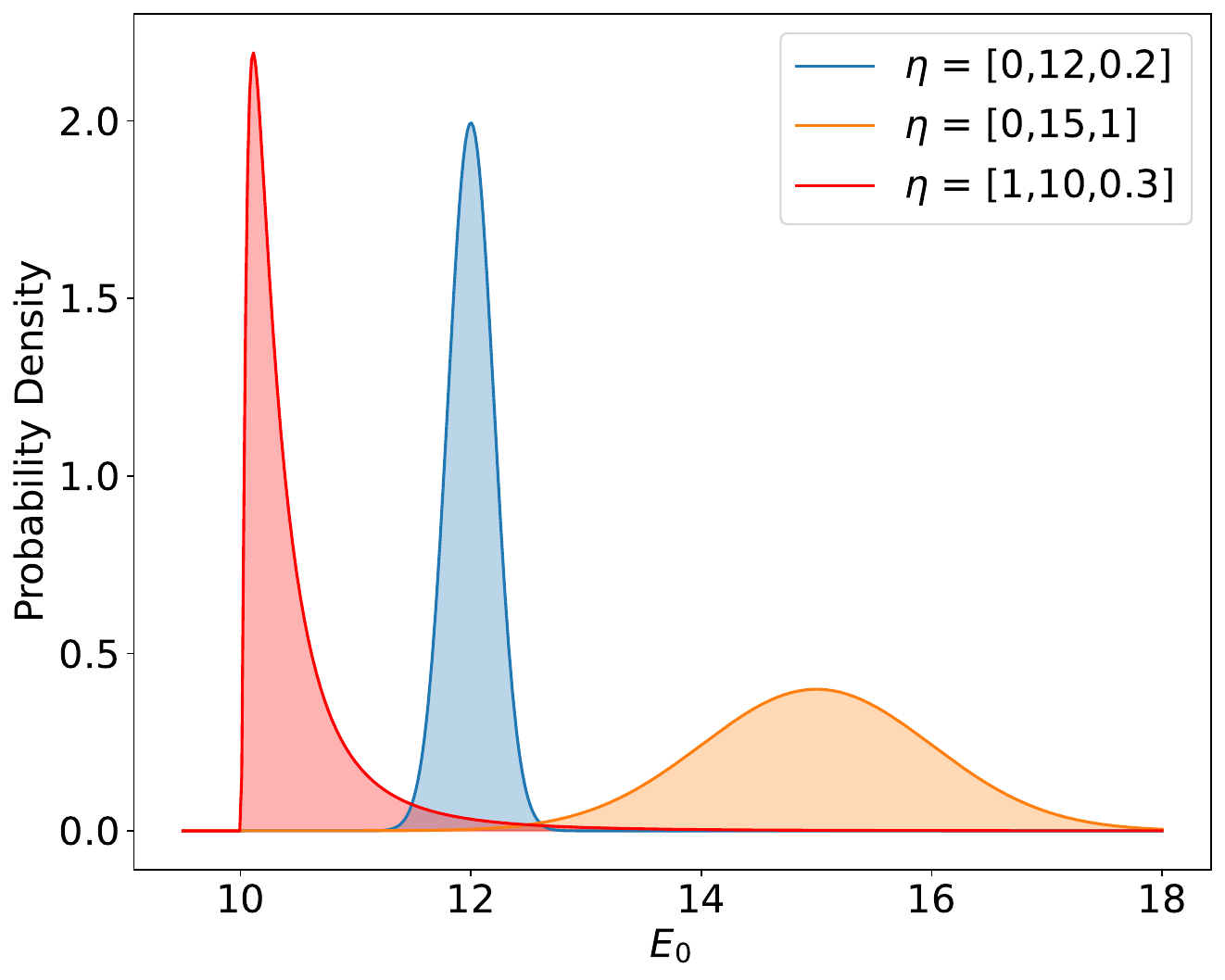}
        \label{pdfs_eta}
    }
    \hfill
    \subfloat[\textbf{90\% credible interval and median of the emission pathways} for normally distributed $E_0$ with loc.\ $15$ and sd.\ $1$ for several SSP-RCP scenarios.]
    {
        \includegraphics[width=0.45\textwidth]{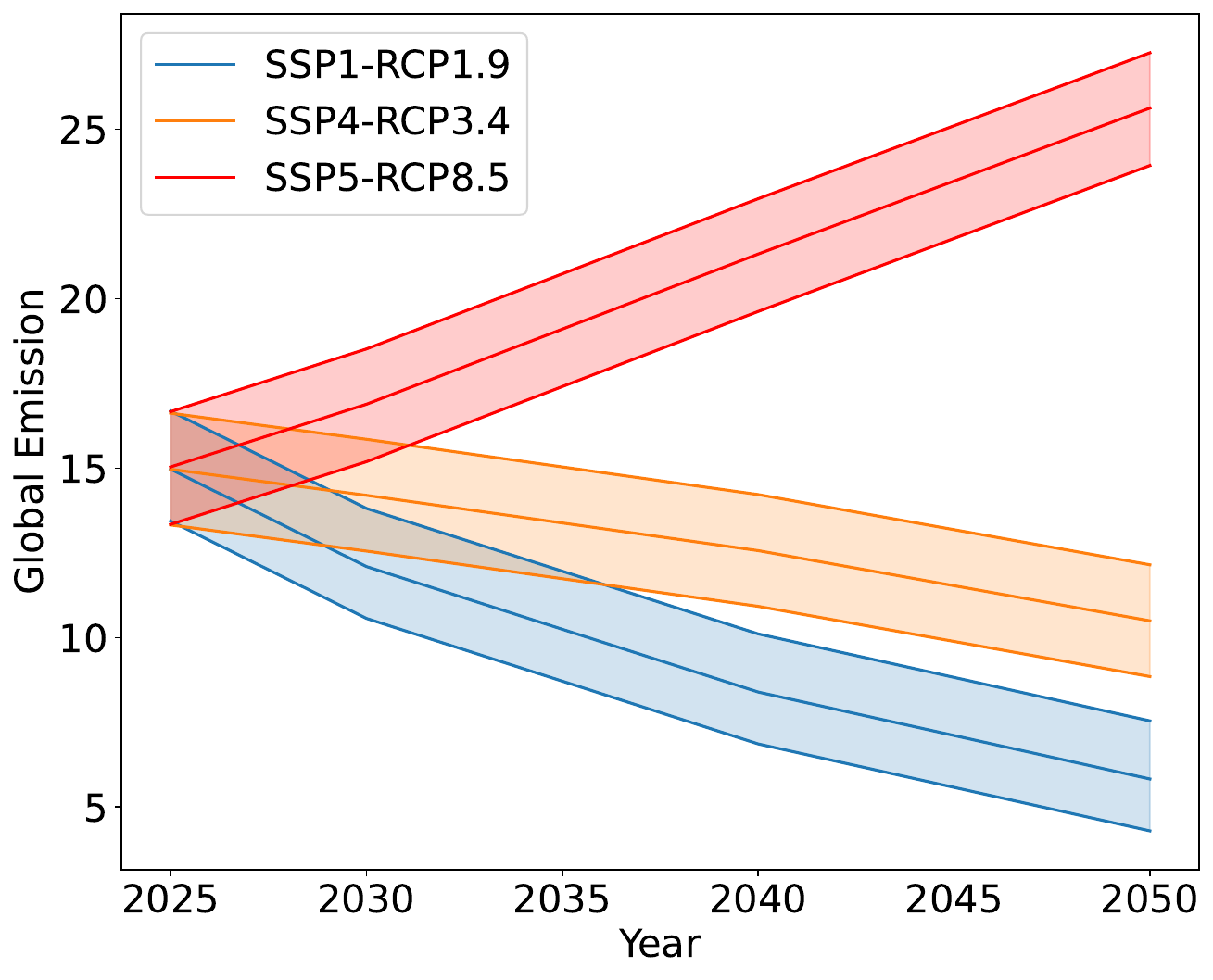}
        \label{fig:eta_paths}
    }
    \caption{\footnotesize Comparison of example $E_0$ distributions and resulting emission pathways.}
    \label{fig:eta_combined}
\end{figure}

Ultimately, the input to the MINE is the vector of hyperparameters of the distribution of base-year emissions, $\boldsymbol{\eta}$, and the socio-economic scenario $s \in S$.  Hence, the distribution of temperature predictions that includes both parameter and emission uncertainty is denoted by $$p(y\mid \boldsymbol{y}_{\text{hist}}, \boldsymbol{x}_{\text{hist}}; \boldsymbol{\eta}, s)=\int\int p(y\mid\boldsymbol{x}=g(s,E_0), \boldsymbol{\theta})\;p(E_0\mid\boldsymbol{\eta})\;p(\boldsymbol{\theta}\mid\boldsymbol{y}_{hist},\boldsymbol{x}_{hist})\;dE_0\;d\boldsymbol{\theta},$$
where $g(s,E_0)$ generates an emission vector based on the scenario $s$ and the base-year emissions $E_0$.  Drawing samples from this predictive posterior via nested Monte Carlo would require sampling from $p(E_0\mid\boldsymbol{\eta})$ and $p(\boldsymbol{\theta}\mid\boldsymbol{y}_{hist}, \boldsymbol{x}_{hist})$ for a substantial number of samples and running FaIR for each sample draw. 

\paragraph{Dataset}
We generate data with the inputs $\boldsymbol{X}\in S\times \mathcal{E}$ by sampling uniformly from $S$, and from $\{0,1\}$ for the distribution type $\boldsymbol{\eta}_{1}$ and within plausible bounds for the location and scale parameters of the initial emission distribution, $\boldsymbol{\eta}_{2}$ and $\boldsymbol{\eta}_{3}$ respectively. We randomly split the data into non-overlapping training, validation, and testing data in the ratio of 6:2:2. The training and validation data are created using a nested Monte Carlo approach and contain single-time projections based on $s$ and the realizations of $E_0$ and $\boldsymbol{\theta}$. That is, we sample $E_0\sim p(E_0\mid\boldsymbol{\eta})$ and $\boldsymbol{\theta}\sim p(\boldsymbol{\theta}\mid\boldsymbol{x}_{hist},\boldsymbol{y}_{hist})$, compute the resulting emission pathway $\boldsymbol{x}$ based on $E_0$ and $s$ and then call the FaIR model to generate the horizon-T prediction $y\sim p(y\mid \boldsymbol{y}_{\text{hist}}, \boldsymbol{x}_{\text{hist}}; \boldsymbol{\eta}, s)$, which corresponds to the global mean temperature in the final year. For the empirical test data, we create a large sample $\widehat{\boldsymbol{y}}^{(1)}, \dots, \widehat{\boldsymbol{y}}^{(M)}$ for each input $\boldsymbol{x}^{(i)}$, $i=1,...,N$ and then calculate the resulting 5\% and 95\% quantiles.

\paragraph{Parameter Setting} 
We again train our network using the Adam optimizer with the learning rate of $1\times10^{-3}$ over 1000 epochs. Since this neural network is very shallow, it can be trained in seconds on an ordinary computer. Only generating train and especially test data is time-consuming; the training dataset had a size of 350,000, and the validation and test datasets had a size of 75,000 each. 

\subsection{Emulating the Chemical Kinetics Model with the ODE Forward Emulator}

The ODE forward emulator for the chemical kinetics has the inputs of the initial concentration of $A,B,C,D$ and $E$ and the output of concentration values for all chemicals at discrete time points $t_i=0,1,\dots,T$.

\paragraph{Datasets}
We collected data based on the chemical data from the chemical system introduced in Section \ref{sec:preliminaries}. We randomly split the data into non-overlapped training, validation and testing datasets. To standardize the data for neural network training, we normalize all data by dividing the maximum value, by approximately 6.45. For the chemical reaction parameters, we take the maximum and minimum values to normalize them to [-1, 1]. To further increase the data variety, we apply data augmentation to the training set. Specifically, we randomly roll the time evolution of the observation as $\hat{y}_i^t = y_i^{t+\tau}, i=1,2,...,100, t=1,2,...,11$, where 

\begin{small}

\begin{equation} 
\hat{y}_i^t =
\left\{
\begin{array}{rcl}
y_i^{t+\tau}       & \text{for} & \tau < 11-t \\[2mm]
y_i^{t+\tau-11}    & \text{for} & \tau > 11-t
\end{array}
\right.
\quad \text{where } i=1,2,\dots,100,\ t=1,2,\dots,11
\label{eq:roll}
\end{equation}
\end{small}

\paragraph{Parameter setting} 
We train \net using Adam optimizer with the learning rate of $1\times10^{-3}$. The batch size is set to 4096 and \net is trained for 50k iterations (about 2 hours) on a PC with one NVIDIA V100 GPU using PyTorch deep learning platform. The weighting factors in the total loss are defined empirically as: $\alpha_1=1, \alpha_2=10, \alpha_3=10, \alpha_4=1, \alpha_5=0.001$.

\subsection{Emulating the FaIR Model with the ODE Forward Emulator}

The ODE forward emulator for the FaIR model has the inputs of $E_0$ and $s$, which are the inputs of the parametrization $g(E_0, s)$ for the input emission $E_{GHG}(t)$ as described in Section~\ref{subsec:fair}, and the model parameters $\boldsymbol{\theta}$. The output of the forward operator is the pathway of annual global mean temperatures $T(t)$.

\paragraph{Dataset.}
We created the training data of size 50,000 by randomly selecting initial emissions $E_0$ from a plausible interval, randomly selecting an SSP-RCP scenario $s$, and then selecting a random parameter $\boldsymbol{\theta}^{(i)}$ from the MCMC chain. Then, we created the temperature prediction $y$ using the FaIR model. Again, we split it with the ratio 6:2:2.

\paragraph{Parameter Setting}
We train AEODE using the same parameter setting as done in the chemical kinetics example. 

\section{Results}
\label{sec:results}

This section provides an explanation of the experiment results. The first part describes the experiments done on the FaIR climate model with the Quantile Emulator, with a focus on validating empirical quantiles. The second part benchmarks the novel \net against two state-of-the-art methods and conducts ablation studies on its key modules, both is done on the chemical kinetics example. To show that \net is also applicable for more complex physical models, we show the results of applying it to the FaIR climate model. 

\subsection{Results for the FaIR Climate Model Quantile Emulator}

\begin{figure}[pos=b]
    \centering
    \includegraphics[width=.8\linewidth]{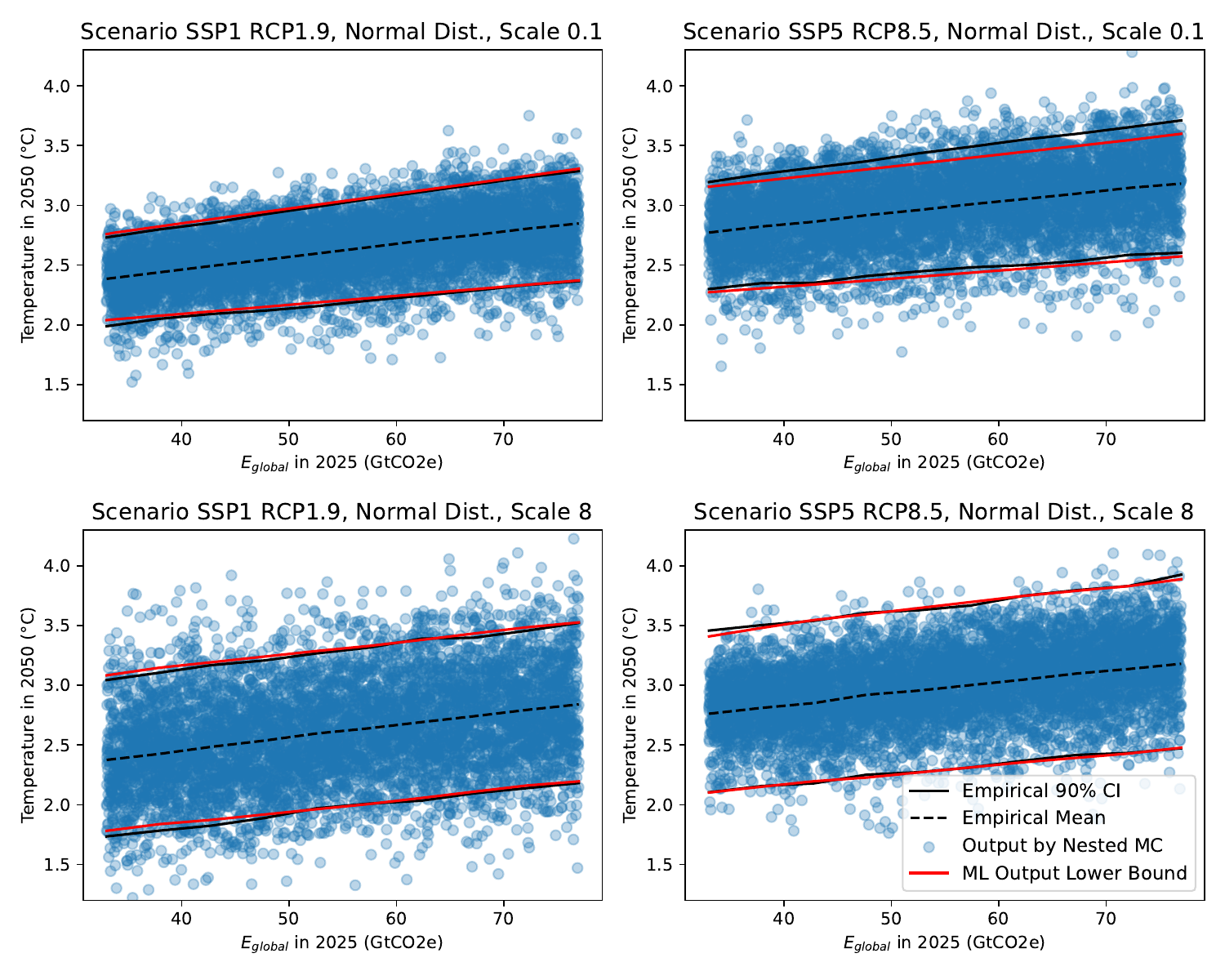}    \caption{\textbf{Visualization of the fit of the empirical and trained 90\% credible interval to the temperature projection.}}
    \label{fig:fair_empirical_trained}
\end{figure}

In Figure \ref{fig:fair_empirical_trained}, one can see the empirical and the trained credible intervals as well as the empirical mean on the sampled temperature distribution for varying values of the expected global emission $E_{\textrm{global}}$ on the x-axis. The two plots on the left side show a positive emission scenario leading to lower temperatures than the right-hand side plots. The two plots on the top show a lower scale and therefore less uncertainty in the emission, leading to smaller credible intervals than the two plots on the bottom. We can see that the empirical results align very well with those produced by the quantile emulator.
The evaluation metrics showed that the MSE for the upper bound of the credible interval was 0.0049, and the lower bound was 0.0023. The pinball loss of test data was 0.0238 and 0.0199, respectively. The mean 90\% coverage, i.e., the average proportion of elements in the set of Monte Carlo predictive draws $\{\widehat y^{(m)}\}_{m=1}^{M}$ that lie within the bounds predicted by the quantile emulator, was 90.04\%. The average interval size of the empirical test data was 1.2907, and the average interval size of the Neural Network was 1.3024, scoring an absolute difference of 0.0117. Comparing the run time of both approaches, we can see the actual benefits of a surrogate model: the feed-forward network quantile emulator has a run time of 0.0006 seconds; the empirical sampling and evaluation has an effective run time of roughly 20 seconds for a sample size of 5,000.

\subsection{Results for the Chemical Kinetics Model \net Forward Emulator}
\subsubsection{Overall comparison with state of the art}
To demonstrate the efficiency of our proposed \net, we compare it with two state-of-the-art methods: ChemiODE~\cite{chemiode}, and Torchdiffeq~\cite{neuralode}. The results are depicted in Table~\ref{tab:test_comparison}. We can see that using our approach achieves the lowest scores in all metrics. Torchdiffeq uses the least running time but also achieves the worst MSE value. On the other hand, our approach uses 15\% extra computation overhead but can significantly reduce the MSE and MBE scores by 5\% and 25\%, respectively. 

\begin{table}[pos=h]
\centering
\begin{tabular}{c|ccccc}
\toprule
\multirow{2}{*}{Modules} & \multicolumn{5}{c}{Evaluation} \\
 & MSE($10^{-5}$) & RMSE ($10^{-3}$) & MAE ($10^{-3}$) & MBE ($10^{-4}$) & Running time ($10^{-5}$ s) \\ \midrule
Torchdiffeq & 3.231 & 5.426 & 3.552 & -7.539 & {\color{red}5.1} \\
ChemiODE & 2.582 & 5.085 & 3.356 & -1.759 & 8.2 \\
Ours & {\color{red}2.450} & {\color{red}4.950} & {\color{red}3.064} & {\color{red}-1.316} & 9.5 \\ \bottomrule
\end{tabular}
\caption{\textbf{Testing data comparison among Torchdiffeq, ChemiODE, and ours.} We report the results on the testing dataset using different methods. For all metrics, the smaller the values, the better the performance we get.}
\label{tab:test_comparison}
\end{table}

\begin{figure}[pos=h]
	\centering
		\centerline{\includegraphics[width=.85\textwidth]{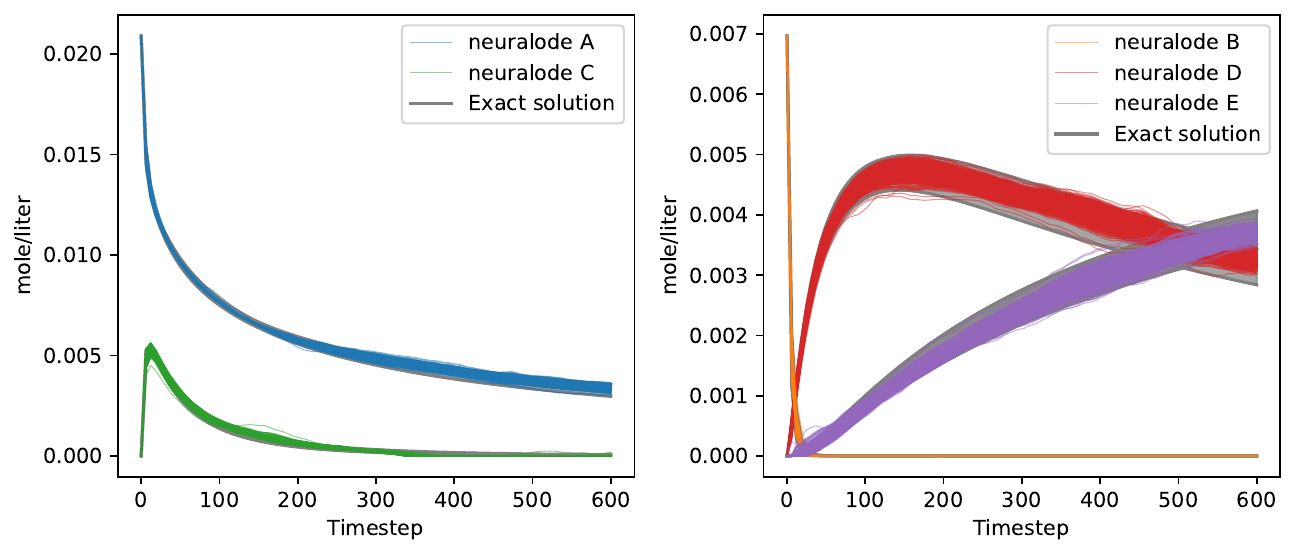}}
		\caption{\small{\textbf{The proposed \net for chemical prediction.} We show all 3000 testing data samples using different chemical reaction parameters in gray lines. The five lines with different colors represent the chemical compounds predicted by the proposed network.
		}}
		\label{fig:ode_compare}
\end{figure}

\begin{figure}[pos=h]
	\centering
		\centerline{\includegraphics[width=.75\textwidth]{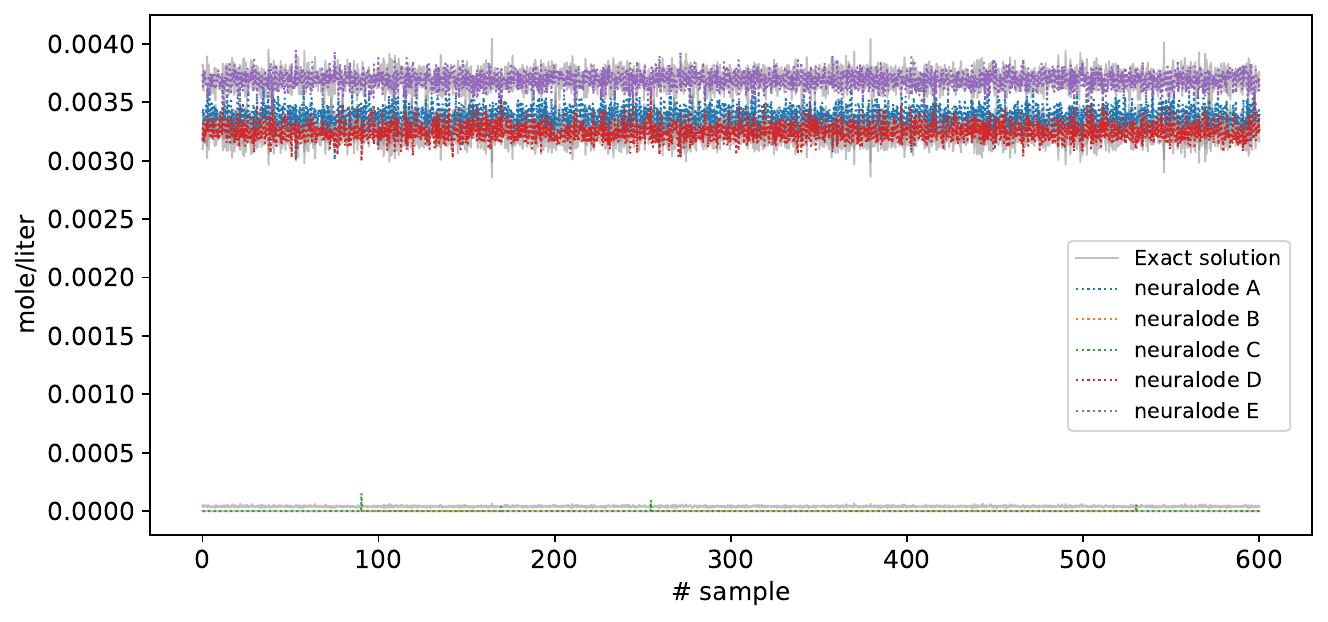}}
		\caption{\small{\textbf{The distribution comparison between the proposed model and the ground truth data.} We show all 3000 testing data samples at the last time step (600) using different chemical reaction parameters in gray lines. The model predictions are shown in different colour, which indicate different chemical compounds.
		}}
		\label{fig:dis_compare}
\end{figure}

Figure~\ref{fig:ode_compare} shows the results of using our proposed \net for chemical ODE prediction. Taken different chemical reaction parameters ${\theta_1, \theta_2, \theta_3}$ with the same initial concentration values, our model is able to estimate the uncertainty of the chemical ODE system for different chemical compounds. For comparison, we show the ground truth chemical values in gray color. We can see that our model can accurately estimate the distribution of chemical changes. For example, at further time steps, the range of concentration distributions becomes wider, and the proposed \net can capture the trend with similar distribution changes. To better visualize the model performance, we slice the chemical concentration values at the last time step and visualize the data distributions. In Figure~\ref{fig:dis_compare}, we can see that ours can estimate the lower and upper bound of the concentration values and align well with the ground truth data.

Figure~\ref{fig:dist_99} shows the distribution differences between ground truth and our model prediction. Similar to Figure~\ref{fig:dis_compare}, we clip the last time step (t=600) across all 3000 testing data samples, and visualize the distribution with the bin size of 256. We can see that the proposed model estimates similar mean and variance as the ground truth. There is a spike for chemical compound C, which our model cannot reproduce. This remains a challenge that would be resolved in the next step.

\subsubsection{Ablation Studies on the Key Modules}
In the Section~\ref{subsec:AEODE}, we propose to use time embedding, self-attention, and physical losses to optimize our network. In order to demonstrate their effects on the overall chemical prediction, we conduct the ablation studies and report the results in Table~\ref{tab:ablation_chem}. NeuralODE is the baseline model for comparison. We add three novel modules one by one to the network and evaluate their performance on the testing dataset. We can see that Time embedding has a significant impact on the performance, approximately $5\times 10^{-6}$ in MSE. Attention improves the MSE by $1\times 10^{-5}$

\begin{table}[pos=h]
\centering

\begin{tabular}{cccc|cccc}
\toprule
\multicolumn{4}{c|}{Modules} & \multicolumn{4}{c}{Evaluation} \\ \midrule
NeuralODE & Time encoding & Attn & Physical loss & MSE($10^{-5}$) & RMSE ($10^{-3}$) & MAE ($10^{-3}$) & MBE ($10^{-4}$) \\ \midrule
\checkmark &  &  &  & 3.101 & 5.569 & 3.565 & -8.190 \\
\checkmark & \checkmark &  &  & 2.593 & 5.092 & 3.209 & {\color{blue}-0.951} \\
\checkmark & \checkmark & \checkmark &  & {\color{blue}2.466} & {\color{blue}4.966} & {\color{blue}3.076} & -1.273 \\
\checkmark & \checkmark & \checkmark & \checkmark & {\color{red}2.450} & {\color{red}4.950} & {\color{red}3.064} & {\color{red}-1.316} \\ \bottomrule
\end{tabular}%

\caption{\textbf{Ablation studies of using proposed modules in the chemical kinetics example.} We report the results on the testing dataset using different key modules and loss terms.}
\label{tab:ablation_chem}
\end{table}

\begin{figure}[pos=t]
	\centering
		\centerline{\includegraphics[width=\textwidth]{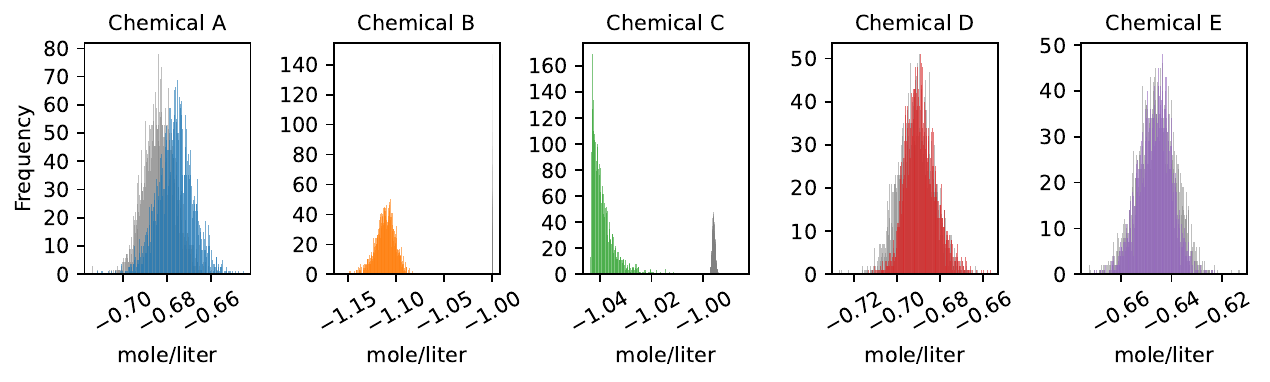}}
		\caption{\small{\textbf{Visualization of the histogram of chemical concentration at time step 600.} We use the histogram to statistically compute the distribution of all 3000 testing data samples at the last time step. The gray curve is the ground truth, and colored lines are model predictions.
		}}
		\label{fig:dist_99}
\end{figure}

\subsection{Results for the FaIR Climate Model \net Forward Emulator}
The proposed \net can also be modified for the FaIR climate model. Different from the chemical kinetics model, the only change is to fit the network to output a 1-D time sequence to estimate the temperature values. 

In Table~\ref{tab:ablation_fair}, we repeat the same ablation comparison done in Table~\ref{tab:ablation_chem}. We compare the models with and without the proposed key modules and evaluate their performance on the FaIR example. We can see that combining time embedding, attention, and physical losses can improve the MSE loss by 21\%. 

To visualize the ODE trajectories of the temperature changes, we show the results in Figure~\ref{fig:forward_Fair_1} and \ref{fig:forward_Fair_2}. As can be seen from Figure~\ref{fig:forward_Fair_1}, the model predictions (dotted lines) align well with the ground truth (straight lines), which indicates that the model can approximate the temperature trend with given conditional input information. In Figure~\ref{fig:forward_Fair_2}, we visualize the overall statistics of all testing samples via mean and variance in every time step. Note that the credible interval is larger than in Figure~\ref{fig:fair_data_fit} uncertainty since this is a future prediction which naturally contains more uncertainty.

\begin{table}[pos=H]
\centering
\begin{tabular}{cccc|cccc}
\toprule
\multicolumn{4}{c|}{Modules} & \multicolumn{4}{c}{Evaluation} \\ \midrule
NeuralODE & Time encoding & Attn & Physical loss & MSE($10^{-5}$) & RMSE ($10^{-3}$) & MAE ($10^{-3}$) & MBE ($10^{-4}$) \\ \midrule
\checkmark &  &  &  & 3.200 & 1.859 & 1.596 & -0.166 \\
\checkmark & \checkmark &  &  & 3.010 & 1.768 & 1.446 & {\color{blue}-0.016} \\
\checkmark & \checkmark & \checkmark &  & {\color{blue}3.002} & {\color{blue}1.751} & {\color{blue}1.423} & -0.015 \\
\checkmark & \checkmark & \checkmark & \checkmark & {\color{red}2.949} & {\color{red}1.718} & {\color{red}1.389} & {\color{red}-0.007} \\ \bottomrule
\end{tabular}

\caption{\textbf{Ablation studies of using proposed modules in the FaIR example.} We report the results on the testing dataset using different key modules and loss terms.}
\label{tab:ablation_fair}
\end{table}

\begin{figure}[pos=H]
    \centering
    \subfloat[\textbf{Ground Truth and Prediction of the ODE forward operator} for FaIR for different scenarios.]
    {
        \includegraphics[width=0.47\textwidth]{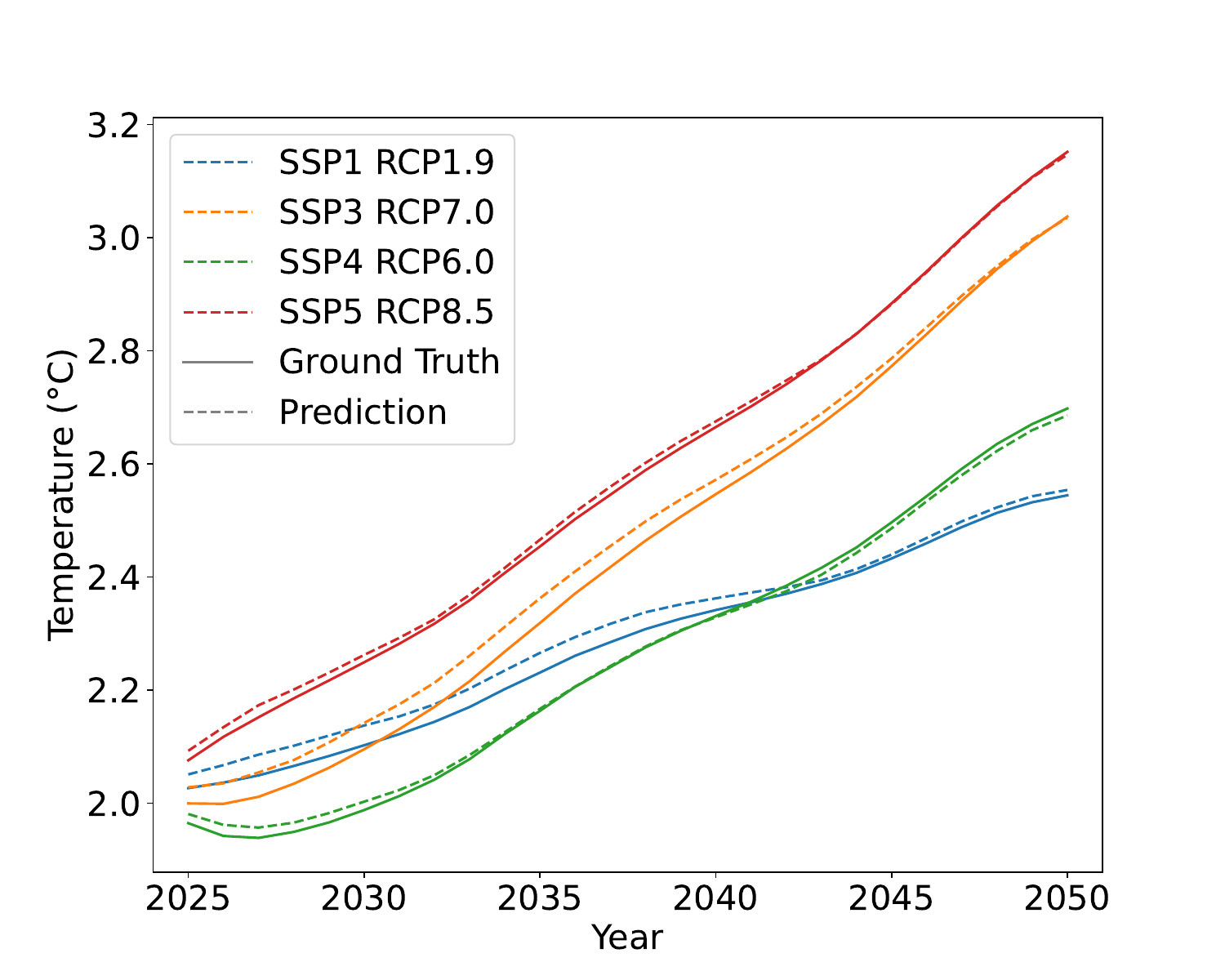}
        \label{fig:forward_Fair_1}
    }
    \hfill
    \subfloat[\textbf{Emulation of 400 FaIR runs by the ODE forward operator with random $\boldsymbol{\theta}$ in each iteration (blue) and 90\% credible interval (red)} for the SSP3 RCP7.0 Scenario.]
    {
        \includegraphics[width=0.47\textwidth]{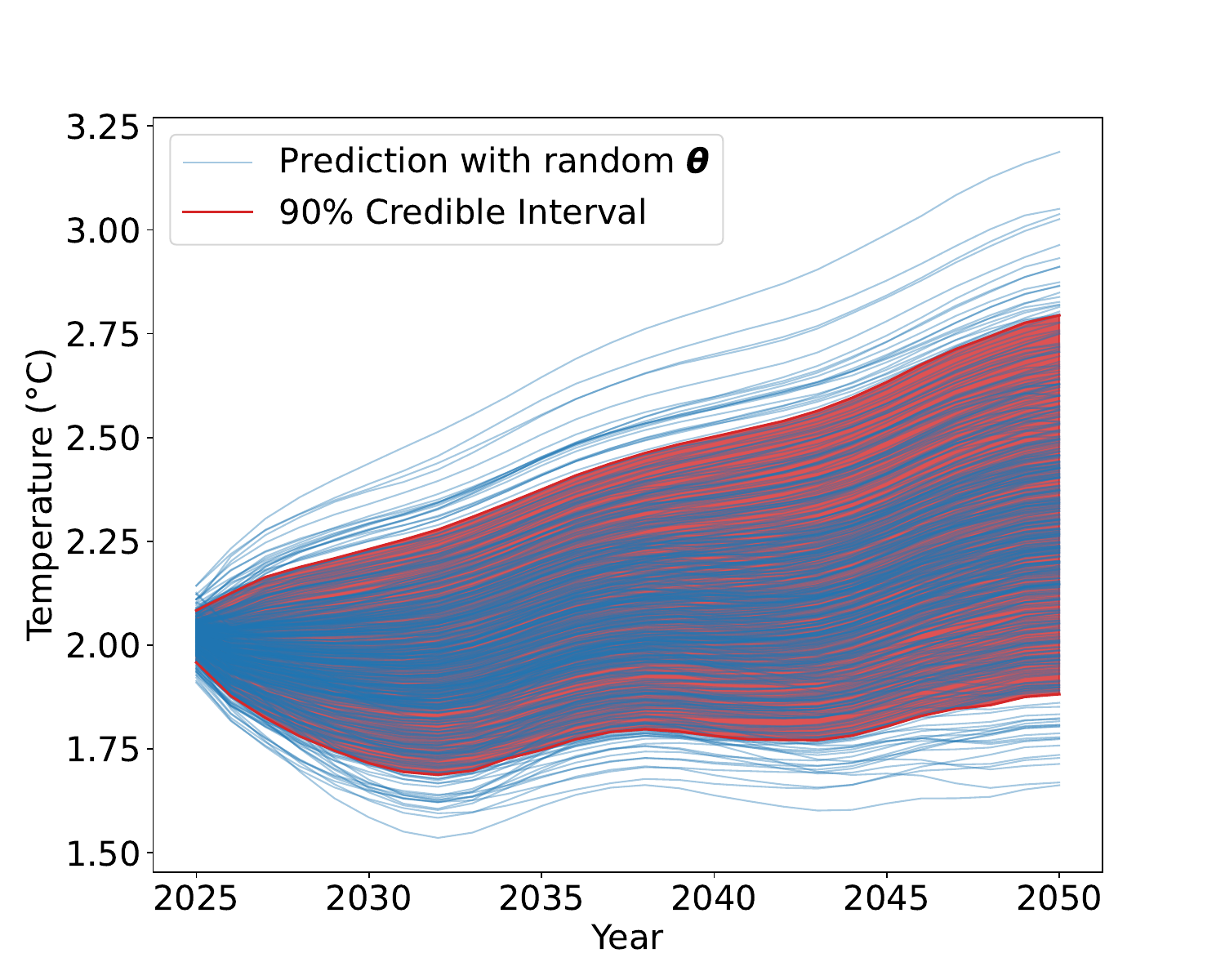}
        \label{fig:forward_Fair_2}
    }
    \caption{Comparison of forward operator simulations.}
    \label{fig:forward_Fair_all}
\end{figure}

\section{Conclusion}
\label{sec:conclusion}

We introduced the \emph{MCMC Informed Neural Emulator} (MINE) paradigm for uncertainty quantification in dynamical systems. The key idea is to decouple Bayesian inference from surrogate modeling: MCMC is used offline to infer a posterior over physical model parameters, and deterministic neural emulators are then trained on posterior-informed input--output pairs. This strategy concentrates training effort on statistically relevant regions of parameter space and avoids expensive simulator evaluations at implausible parameter values. We presented two complementary MINE realizations. First, a quantile (interval) emulator learns posterior predictive quantiles directly, enabling low-latency uncertainty intervals without sampling at inference time. Second, a forward emulator learns the simulator mapping and enables efficient posterior predictive sampling by reusing stored posterior draws of parameters. For the forward emulator, we instantiated MINE using an AutoEncoder-based ODE neural network (AEODE) with time embeddings and attention, and we evaluated performance using both accuracy and calibration-oriented metrics.

Empirically, the proposed approach yields accurate and computationally efficient uncertainty-aware predictions on both a chemical kinetics ODE system and the FaIR simple climate model. In particular, the quantile emulator closely matches empirical credible intervals for FaIR temperature projections, while the AEODE forward emulator achieves favorable accuracy-runtime trade-offs compared with existing neural-ODE baselines and benefits from the proposed architectural components in ablation studies. Several limitations and directions for future work remain. MINE assumes that offline MCMC sampling is feasible; extending the framework to settings with extremely expensive simulators may require multi-fidelity strategies and/or surrogate-assisted sampling. Finally, beyond methodological extensions, we envision MINE serving as a practical uncertainty-aware surrogate for scientific simulators in climate science, atmospheric chemistry, and energy systems, enabling faster hypothesis testing, scenario exploration, and data assimilation in computationally intensive modeling pipelines.

\pagebreak
\appendix
\section{Proofs of the results in Section \ref{sec:MCMCInformedTraining}}
\label{app:mathanalysis}

\begin{proof}[Proof of Lemma \ref{lem:shift}]
Let $\Delta(u):=F(u)-\mathcal{E}(u)\in\mathcal{Y}$ and $h(u):=\|\Delta(u)\|_{\mathcal{Y}}^2$. For arbitrary $u,v\in\mathcal{U}$,
$$
|h(u)-h(v)|
=|\langle\Delta(u)+\Delta(v),\,\Delta(u)-\Delta(v)\rangle_{\mathcal{Y}}|
\le \big(\|\Delta(u)\|_{\mathcal{Y}}+\|\Delta(v)\|_{\mathcal{Y}}\big)\,\|\Delta(u)-\Delta(v)\|_{\mathcal{Y}}.
$$
By Lipschitz continuity of $F$ and $\mathcal{E}$ we have
$$
\|\Delta(u)-\Delta(v)\|_{\mathcal{Y}}
\le \|F(u)-F(v)\|_{\mathcal{Y}}+\|\mathcal{E}(u)-\mathcal{E}(v)\|_{\mathcal{Y}}
\le (L+R)\,\|u-v\|_{\mathcal{U}}=C_1\,\|u-v\|_{\mathcal{U}}.
$$
Moreover,
$$
\|\Delta(w)\|_{\mathcal{Y}}
\le \|F(w)-F(0)\|_{\mathcal{Y}}+\|\mathcal{E}(w)-\mathcal{E}(0)\|_{\mathcal{Y}}+\|F(0)\|_{\mathcal{Y}}+\|\mathcal{E}(0)\|_{\mathcal{Y}}
\le C_1\|w\|_{\mathcal{U}}+C_2.
$$
Hence,
\begin{equation}
\label{eq:increment}
|h(u)-h(v)|\ \le\ \big(C_1(\|u\|_{\mathcal{U}}+\|v\|_{\mathcal{U}})+2C_2\big)\,C_1\,\|u-v\|_{\mathcal{U}}.
\end{equation}
Let $\Gamma(\nu,\nu')$ denote the set of couplings of $\nu$ and $\nu'$. For any $\gamma\in\Gamma(\nu,\nu')$,
$$
\Big|\mathbb{E}_{\nu'}h-\mathbb{E}_{\nu}h\Big|
=\Big|\int_{\,\mathcal{U}\times\mathcal{U}}\big(h(v)-h(u)\big)\,d\gamma(u,v)\Big|
\le \int \big|h(v)-h(u)\big|\,d\gamma(u,v).
$$
Using \eqref{eq:increment}, Cauchy–Schwarz, and the elementary bound $(a+b)^2\le 2(a^2+b^2)$,
\begin{align*}
\int \big|h(v)-h(u)\big|\,d\gamma
&\le C_1^2 \int (\|u\|_{\mathcal{U}}+\|v\|_{\mathcal{U}})\,\|u-v\|_{\mathcal{U}}\,d\gamma 
      \;+\; 2C_1C_2 \int \|u-v\|_{\mathcal{U}}\,d\gamma\\
&\le C_1^2 \bigg(\int (\|u\|_{\mathcal{U}}+\|v\|_{\mathcal{U}})^2 d\gamma\bigg)^{1/2}
               \bigg(\int \|u-v\|_{\mathcal{U}}^2 d\gamma\bigg)^{1/2}\\
     &\phantom{le} \;+\; 2C_1C_2 \bigg(\int \|u-v\|_{\mathcal{U}}^2 d\gamma\bigg)^{1/2}\\
&\le C_1^2 \sqrt{\,2\Big(\mathbb{E}_{\nu}\|u\|_{\mathcal{U}}^2+\mathbb{E}_{\nu'}\|u\|_{\mathcal{U}}^2\Big)}\,
               \bigg(\int \|u-v\|_{\mathcal{U}}^2 d\gamma\bigg)^{1/2}\\
&\phantom{le}
      +\; 2C_1C_2 \bigg(\int \|u-v\|_{\mathcal{U}}^2 d\gamma\bigg)^{1/2}.
\end{align*}
Now take the infimum over $\gamma\in\Gamma(\nu,\nu')$. By the Kantorovich formulation of $W_2$ (see, e.g., Villani~\cite{villani2009optimal}),
$$
\inf_{\gamma\in\Gamma(\nu,\nu')}\bigg(\int \|u-v\|_{\mathcal{U}}^2 d\gamma\bigg)^{1/2}
= W_2(\nu,\nu').
$$
Therefore,
$$
\Big|\mathbb{E}_{\nu'}h-\mathbb{E}_{\nu}h\Big|
\le \bigg( C_1^2 \sqrt{\,2\big(\mathbb{E}_{\nu}\|u\|_{\mathcal{U}}^2+\mathbb{E}_{\nu'}\|u\|_{\mathcal{U}}^2\big)} + 2C_1C_2 \bigg)\, W_2(\nu,\nu')
= c(\nu,\nu')\,W_2(\nu,\nu').
$$
Recalling $h(u)=\|F(u)-\mathcal{E}(u)\|_{\mathcal{Y}}^2$ yields the claim:
$$
\mathcal{R}_{\nu'}(\mathcal{E})-\mathcal{R}_{\nu}(\mathcal{E})
=\mathbb{E}_{\nu'}h-\mathbb{E}_{\nu}h
\le c(\nu,\nu')\,W_2(\nu,\nu').
$$
\end{proof}
\begin{remark}
(i) The proof of Lemma \ref{lem:shift} is a standard “coupling + Cauchy–Schwarz” argument; see also Santambrogio~\cite{santabrogio2015optimal} for the $W_1$–version with Lipschitz test functions and Bolley–Villani~\cite{bolley2005weighted} for related inequalities with linear growth. (ii) The constants enter only through Lipschitz moduli and the second moments of $(\nu,\nu')$, hence the restriction to $\mathcal{P}_2(\mathcal{U})$.
\end{remark}

\begin{proof}[Proof of Proposition \ref{prop:opt}]
Apply Lemma~\ref{lem:shift} with $(\nu,\nu')=(\nu,\nu_{\mathrm{dep}})$: For any $\mathcal{E}\in\mathcal{H}$,
\(
\mathcal{R}_{\nu_{\mathrm{dep}}}(\mathcal{E})
\le \mathcal{R}_{\nu}(\mathcal{E}) + c(\nu,\nu_{\mathrm{dep}}) W_2(\nu,\nu_{\mathrm{dep}}).
\)
Taking $\inf_{\mathcal{E}}$ on the right and adding the same penalty gives
\(
J(\nu)\ge \inf_{\mathcal{E}}\mathcal{R}_{\nu_{\mathrm{dep}}}(\mathcal{E}) = J(\nu_{\mathrm{dep}}),
\)
with equality at $\nu=\nu_{\mathrm{dep}}$ since $W_2(\nu_{\mathrm{dep}},\nu_{\mathrm{dep}})=0$.
For the quantitative upper bound, let $\mathcal{E}^\dagger\in\arg\inf_{\mathcal{E}}\mathcal{R}_{\nu_{\mathrm{dep}}}(\mathcal{E})$ and use Lemma~\ref{lem:shift} in both directions to obtain
\begin{align*}
J(\nu)\le \mathcal{R}_{\nu}(\mathcal{E}^\dagger)+c(\nu,\nu_{\mathrm{dep}})W_2(\nu,\nu_{\mathrm{dep}})
&\le& \mathcal{R}_{\nu_{\mathrm{dep}}}(\mathcal{E}^\dagger)+2\,c(\nu,\nu_{\mathrm{dep}})W_2(\nu,\nu_{\mathrm{dep}})\\
&=& J(\nu_{\mathrm{dep}})+2\,c(\nu,\nu_{\mathrm{dep}})W_2(\nu,\nu_{\mathrm{dep}}).
\end{align*}
\end{proof}

\begin{proof}[Proof of Lemma \ref{lem:W2_product_reduction}]
Let $\gamma_\Theta\in\Gamma(\pi,\widehat\pi)$ be optimal for $W_2(\pi,\widehat\pi)$ and let
$\gamma_X$ be the diagonal coupling of $\rho$ with itself, i.e.,
$\gamma_X(d\boldsymbol{x},d\boldsymbol{x}')=\rho(d\boldsymbol{x})\,\delta_{\boldsymbol{x}}(d\boldsymbol{x}')$.
Then $\gamma:=\gamma_X\otimes\gamma_\Theta$ lies in
$\Gamma(\rho\otimes\pi,\rho\otimes\widehat\pi)$ and
$$\int_{\mathcal U\times\mathcal U}\|(\boldsymbol{x},\boldsymbol{\theta})-(\boldsymbol{x}',\boldsymbol{\theta}')\|_{\mathcal U}^2\,d\gamma
=
\int_{\Theta\times\Theta}\|\boldsymbol{\theta}-\boldsymbol{\theta}'\|_\Theta^2\,d\gamma_\Theta,
$$
so $W_2(\rho\otimes\pi,\rho\otimes\widehat\pi)\le W_2(\pi,\widehat\pi)$.

Conversely, take any $\gamma\in\Gamma(\rho\otimes\pi,\rho\otimes\widehat\pi)$ and let $\gamma_\Theta$
be its $(\boldsymbol{\theta},\boldsymbol{\theta}')$-marginal. Then $\gamma_\Theta\in\Gamma(\pi,\widehat\pi)$ and, since the $x$-term in the
product cost is nonnegative,
$$\int_{\mathcal U\times\mathcal U}\|(\boldsymbol{x},\boldsymbol{\theta})-(\boldsymbol{x}',\boldsymbol{\theta}')\|_{\mathcal U}^2\,d\gamma
\ge
\int_{\Theta\times\Theta}\|\boldsymbol{\theta}-\boldsymbol{\theta}'\|_\Theta^2\,d\gamma_\Theta
\ge
W_2(\pi,\widehat\pi)^2.
$$
Taking the infimum over $\gamma$ gives the reverse inequality, hence equality.
\end{proof}

\begin{proof}[Proof of Theorem \ref{thm:finite_chain}]
\emph{Step 1: Finite $N$ bound and the $W_2$ reduction.}
Apply Proposition~\ref{prop:opt} with deployment law $\nu_{\mathrm{dep}}:=\nu$ and training law $\widehat\nu_N$:
$$
0\le J_{\nu}(\widehat\nu_N)-J_{\nu}(\nu)
\le 2\,c(\widehat\nu_N,\nu)\,W_2(\widehat\nu_N,\nu).
$$
Lemma~\ref{lem:W2_product_reduction} yields $W_2(\widehat\nu_N,\nu)=W_2(\widehat\pi_N,\pi)$, proving
\eqref{eq:finite_chain_bound}.

\emph{Step 2: $W_2(\widehat\pi_N,\pi)\to0$.}
Since $\Theta$ is a separable Hilbert space, it is Polish. Let $\mathcal G\subset C_b(\Theta)$ be a
\emph{countable} convergence-determining class (e.g.\ a countable dense subset of bounded Lipschitz test
functions). By the assumed ergodic theorem, for each $g\in\mathcal G$,
$$
\int g\,d\widehat\pi_N
=\frac1N\sum_{i=1}^N g(\boldsymbol{\theta}^{(i)})
\longrightarrow
\int g\,d\pi
\quad\text{a.s.}
$$
By countability, these convergences hold simultaneously a.s.\ for all $g\in\mathcal G$, which implies
$\widehat\pi_N\Rightarrow \pi$ weakly. 
By assumption (ii), the empirical second moment converges almost surely:
$$
\int \|\boldsymbol{\theta}\|_\Theta^2\,d\widehat\pi_N(\boldsymbol{\theta})
=\frac1N\sum_{i=1}^N \|\boldsymbol{\theta}^{(i)}\|_\Theta^2
\longrightarrow
\int \|\boldsymbol{\theta}\|_\Theta^2\,d\pi(\boldsymbol{\theta})
\quad\text{a.s.}
$$
On Polish spaces, weak convergence plus convergence of second moments is equivalent to convergence in
$W_2$ (see, e.g., Villani~\cite{villani2009optimal}).

\emph{Step 3: Boundedness and convergence of $c(\widehat\nu_N,\nu)$.}
By construction,
$$
\mathbb E_{\widehat\nu_N}\|\boldsymbol{u}\|_{\mathcal U}^2
=
\mathbb E_{\rho}\|\boldsymbol{x}\|_{\mathcal X}^2
+\int \|\boldsymbol{\theta}\|_\Theta^2\,d\widehat\pi_N(\boldsymbol{\theta}),
$$
so the second-moment convergence above implies
$\mathbb E_{\widehat\nu_N}\|\boldsymbol{u}\|_{\mathcal U}^2\to \mathbb E_{\nu}\|\boldsymbol{u}\|_{\mathcal U}^2$ a.s.
Therefore $c(\widehat\nu_N,\nu)\to c(\nu,\nu)$ and is eventually bounded almost surely.
\end{proof}

\pagebreak
\section{MCMC Chains of the FaIR Parameter Calibration}
\label{app:chains_fair}

Figure \ref{fig:fair_chains} shows the chains of the MCMC simulations for the FaIR model. It was created for 300,000 samples. The burn in phase is 50,000 samples which are not used in the posterior sampling.

\begin{figure}[pos=h]
    \centering
    \includegraphics[width=.95\linewidth]{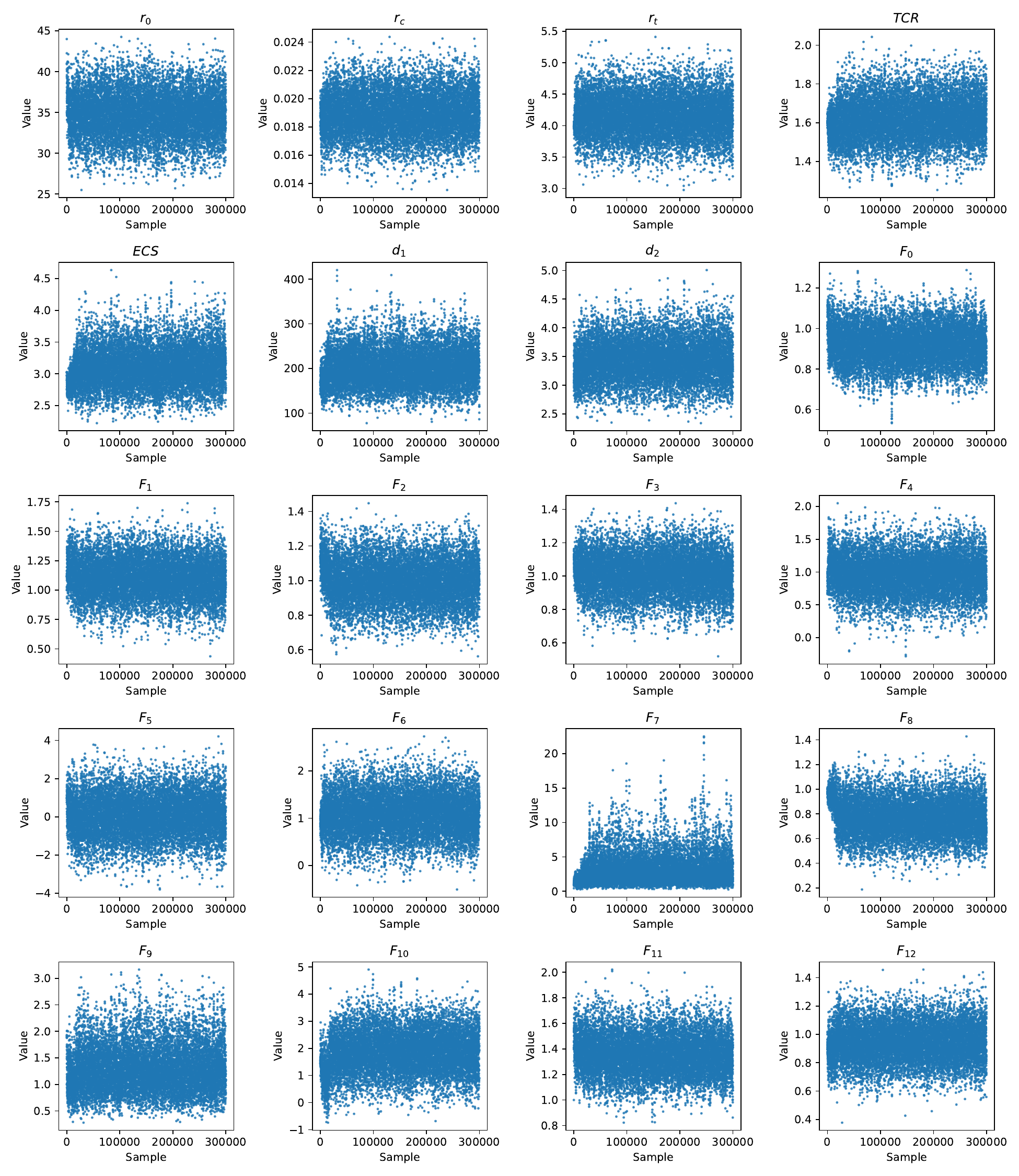}
    \caption{Chains of the MCMC experiments for the FaIR climate model.}
    \label{fig:fair_chains}
\end{figure}

\section*{Declaration on the Use of Generative AI in Scientific Writing}

The authors used ChatGPT to assist with language editing. All resulting content was critically reviewed, edited, and approved by the authors. The authors remain fully responsible for the accuracy, originality, and overall integrity of the manuscript.


\bibliographystyle{plain}

\bibliography{main}

@article{neuralode,
  title={Neural Ordinary Differential Equations},
  author={Chen, R. T. Q. and Rubanova, Y. and Bettencourt, J. and Duvenaud, D.},
  journal={Advances in Neural Information Processing Systems},
  year={2018}
}

@article{neuralode2,
  title={Stabilized neural ordinary differential equations for long-time forecasting of dynamical systems},
  author={Linot, A. J. and Burby, J. W. and Tang, Q. and Balaprakash, P. and Graham, M. D and Maulik, R.},
  journal={Journal of Computational Physics},
  volume={474},
  pages={111838},
  year={2023},
  publisher={Elsevier}
}

@inproceedings{attn1,
 author = {Vaswani, A. and Shazeer, N. and Parmar, N. and Uszkoreit, J. and Jones, L. and Gomez, A. N. and Kaiser, \L. and Polosukhin, I.},
 booktitle = {Advances in Neural Information Processing Systems},
 title = {Attention is All you Need},
 volume = {30},
 year = {2017}
}

@article{attn2,
  title={FunnyNet-W: Multimodal Learning of Funny Moments in Videos in the Wild},
  author={Liu, ZS. and Courant, R. and Kalogeiton, V.},
  journal={International Journal of Computer Vision},
  doi={https://doi.org/10.1007/s11263-024-02000-2},
  year={2024}
}

@article{attn3,
  title={Language models are unsupervised multitask learners},
  author={Radford, Alec and Wu, Jeffrey and Child, Rewon and Luan, David and Amodei, Dario and Sutskever, Ilya and others},
  journal={OpenAI blog},
  volume={1},
  number={8},
  pages={9},
  year={2019}
}

@INPROCEEDINGS{attn4,
  author={Rombach, R. and Blattmann, A. and Lorenz, D. and Esser, P. and Ommer, B.},
  booktitle={2022 IEEE/CVF Conference on Computer Vision and Pattern Recognition (CVPR)}, 
  title={High-Resolution Image Synthesis with Latent Diffusion Models}, 
  year={2022},
  volume={},
  number={},
  pages={10674-10685},
}

@INPROCEEDINGS{chemiode,
  author={Sulzer, I. and Buck, T.},
  booktitle={Machine Learning and the Physical Sciences Workshop, NeurIPS 2023}, 
  title={Speeding up astrochemical reaction networks with
autoencoders and neural ODEs}, 
  year={2023},
  month={12},
  volume={52},
  number={},
}

@article{ode_1,
  title={How to Train Your Neural ODE: the World of Jacobian and Kinetic Regularization},
  author={Finlay, C. and Jacobsen, J. and Nurbekyan, L. and Oberman, A.},
  journal={International Conference on Machine Learning},
  year={2020}
}

@article{ode_2,
  title={Dissecting Neural ODEs},
  author={Massaroli, S. and Poli, M. and Park, J. and Yamashita, A. and Asama, H.},
  journal={Advances in neural information processing systems},
  year={2020}
}

@article{ode_3,
  title={Latent ordinary differential equations for irregularly-sampled time series},
  author={Rubanova, Y. and Chen, R. T. Q. and Duvenaud, D. K.},
  journal={Advances in neural information processing systems},
  volume={32},
  year={2019}
}

@article{ode_4,
  title={ODE2VAE: Deep generative second order ODEs with Bayesian neural networks},
  author={Yildiz, C. and Heinonen, M. and Lahdesmaki, H.},
  journal={Advances in neural information processing systems},
  year={2019}
}

@InProceedings{ode_5,
  title = 	 {Hybrid$^2$ Neural {ODE} Causal Modeling and an Application to Glycemic Response},
  author =       {Zou, B. and Levine, E. and Zaharieva, P. and Johari, R. and Fox, E.},
  booktitle = 	 {Proceedings of the 41st International Conference on Machine Learning},
  pages = 	 {62934--62963},
  year = 	 {2024},
  volume = 	 {235},
  series = 	 {Proceedings of Machine Learning Research},
  month = 	 {21--27 Jul},
  publisher =    {PMLR},
}

@InProceedings{ode_6,
  title = 	 {From Fourier to Neural ODEs: Flow Matching for Modeling Complex Systems},
  author =       {Li, X. and Zhang, J. and Zhu, Q. and Zhao, C. and Zhang, X. and Duan, X. and Lin, W.},
  booktitle = 	 {Proceedings of the 41st International Conference on Machine Learning},
  pages = 	 {29390--29405},
  year = 	 {2024},
  volume = 	 {235},
  series = 	 {Proceedings of Machine Learning Research},
  month = 	 {21--27 Jul},
}

@article{chemode_2,
title = {Neural network emulator for atmospheric chemical ODE},
journal = {Neural Networks},
volume = {184},
pages = {107106},
year = {2025},
issn = {0893-6080},
doi = {https://doi.org/10.1016/j.neunet.2024.107106},
author = {Liu, ZS. and Clusius, P. and Boy, M.},
}

@article{stiff-pinn,
  title={Stiff-PINN: Physics-Informed Neural Network for Stiff Chemical Kinetics},
  author={Ji, W. and Qiu, W. and Shi, Z. and Pan, S. and Deng, S.},
  journal={The Journal of Physical Chemistry A},
  volume = {125},
  number = {36},
  year={2021}
}

@article{mpinn,
title = {A practical PINN framework for multi-scale problems with multi-magnitude loss terms},
journal = {Journal of Computational Physics},
volume = {510},
pages = {113112},
year = {2024},
issn = {0021-9991},
author = {Wang, Y. and Yao, Y. and Guo, J. and Gao, Z.},
}

@article{air_1,
author = {Hou, L. and Dai, Q. and Song, C. and Liu, B. and Guo, F. and Dai, T. and Li, L. and Liu, B. and Bi, X. and Zhang, Y. and Feng, Y.},
title = {Revealing Drivers of Haze Pollution by Explainable Machine Learning},
journal = {Environmental Science \& Technology Letters},
volume = {9},
number = {2},
pages = {112-119},
year = {2022},
doi = {10.1021/acs.estlett.1c00865},
}

@article{air_2,
author = {Betancourt, C. and Li, C. W. Y. and Kleinert, F. and Schultz, M. G.},
title = {Graph Machine Learning for Improved Imputation of Missing Tropospheric Ozone Data},
journal = {Environmental Science \& Technology},
volume = {57},
number = {46},
pages = {18246-18258},
year = {2023},
doi = {10.1021/acs.est.3c05104},
    note ={PMID: 37661931},
}

@article{air_3,
author = {Zhu, Q. and Laughner, J. L. and Cohen, R. C.},
title = {Combining Machine Learning and Satellite Observations to Predict Spatial and Temporal Variation of near Surface OH in North American Cities},
journal = {Environmental Science \& Technology},
volume = {56},
number = {11},
pages = {7362-7371},
year = {2022},
doi = {10.1021/acs.est.1c05636},
    note ={PMID: 35302754},
}

@article{air_4,
author = {Di, Q. and Amini, H. and Shi, L. and Kloog, I. and Silvern, R. and Kelly, J. and Sabath, M. B. and Choirat, C. and Koutrakis, P. and Lyapustin, A. and Wang, Y. and Mickley, L. J. and Schwartz, J.},
title = {Assessing NO2 Concentration and Model Uncertainty with High Spatiotemporal Resolution across the Contiguous United States Using Ensemble Model Averaging},
journal = {Environmental Science \& Technology},
volume = {54},
number = {3},
pages = {1372-1384},
year = {2020},
doi = {10.1021/acs.est.9b03358},
    note ={PMID: 31851499},
}

@article{physics_1,
title = {Physics-informed PointNet: A deep learning solver for steady-state incompressible flows and thermal fields on multiple sets of irregular geometries},
journal = {Journal of Computational Physics},
volume = {468},
pages = {111510},
year = {2022},
issn = {0021-9991},
author = {Kashefi, A. and Mukerji, T.},
}

@article{physics_2,
title = {Multiscale corrections by continuous super-resolution},
journal = {Neural Networks},
volume = {197},
pages = {108516},
year = {2026},
issn = {0893-6080},
doi = {https://doi.org/10.1016/j.neunet.2025.108516},
author = {Liu, ZS. and Maier, R. and Rupp, A.},
}

@misc{physics_3,
      title={Super-Resolution works for coastal simulations}, 
      author={Liu, ZS. and Buttner, M. and Aizinger, V. and Rupp, A.},
      year={2024},
      eprint={2408.16553},
      archivePrefix={arXiv},
      primaryClass={eess.IV},
      url={https://arxiv.org/abs/2408.16553}, 
}

@inproceedings{attn_no_1,
author = {Hao, Z. and Wang, Z. and Su, H. and Ying, C. and Dong, Y. and Liu, S. and Cheng, Z. and Song, J. and Zhu, J.},
title = {GNOT: a general neural operator transformer for operator learning},
year = {2023},
booktitle = {Proceedings of the 40th International Conference on Machine Learning},
articleno = {509},
numpages = {14},
location = {Honolulu, Hawaii, USA},
series = {ICML'23}
}

@article{attn_no_2, 
title={Multiscale Attention Wavelet Neural Operator for Capturing Steep Trajectories in Biochemical Systems}, 
volume={38}, 
number={13}, 
journal={Proceedings of the AAAI Conference on Artificial Intelligence}, 
author={Su, J. and Ma, J. and Tong, S. and Xu, E. and Chen, M.}, 
year={2024}, 
month={Mar.}, 
pages={15100-15107} 
}

@article{finance_1,
title = {Solving high-dimensional partial differential equations using deep learning},
journal = {Proc. Natl. Acad. Sci.},
volume = {115},
pages = {8505-8510},
year = {2018},
no = {34},
author = {Han, J. and Jentzen, A. and E, W.},
}

@misc{finance_2,
      title={Adaptive Movement Sampling Physics-Informed Residual Network (AM-PIRN) for Solving Nonlinear Option Pricing models}, 
      author={Gao, Q. and Wang, Z. and Zhang, R. and Wang, D.},
      year={2025},
      eprint={2504.03244},
      archivePrefix={arXiv},
      primaryClass={cs.NI},
      url={https://arxiv.org/abs/2504.03244}, 
}

@inproceedings{finance_3,
author = {Ibrahim, A. Q. and G\"{o}tschel, S. and Ruprecht, D.},
title = {Parareal with a Physics-Informed Neural Network as coarse propagator},
year = {2023},
isbn = {978-3-031-39697-7},
booktitle = {29th International Conference on Parallel and Distributed Computing},
pages = {649–663},
numpages = {15},
}

@inproceedings{finance_4,
author = {Nuugulu, S. M. and Patidar, K. C. and Tarla, D. T.},
title = {A Physics informed neural network approach for solving time fractional Black-Scholes partial differential equations},
year = {2024},
doi = {https://doi.org/10.1007/s11081-024-09910-7},
booktitle = {Optim Eng},
}

@article{aurora,
  title = {A Foundation Model for the {{Earth}} System},
  author = {Bodnar, C. and Bruinsma, Wessel P. and Lucic, A. and Stanley, M. and Allen, A. and Brandstetter, J. and Garvan, P. and Riechert, M. and Weyn, J. A. and Dong, H. and Gupta, J. K. and Thambiratnam, K. and Archibald, A. T. and Wu, C.-C. and Heider, E. and Welling, M. and Turner, R. E. and Perdikaris, P.},
  year = 2025,
  month = may,
  journal = {Nature},
  volume = {641},
  number = {8065},
  pages = {1180--1187},
  issn = {1476-4687},
  doi = {10.1038/s41586-025-09005-y},
}

@misc{bnn,
      title={Bayesian Neural Networks}, 
      author={Mullachery, V. and Khera, A. and Husain, A.},
      year={2018},
      eprint={1801.07710},
      archivePrefix={arXiv},
      primaryClass={cs.LG},
      url={https://arxiv.org/abs/1801.07710}, 
}

@article{bnn_2,
title = {Bayesian neural networks and density networks},
journal = {Nuclear Instruments and Methods in Physics Research Section A: Accelerators, Spectrometers, Detectors and Associated Equipment},
volume = {354},
number = {1},
pages = {73-80},
year = {1995},
issn = {0168-9002},
author = {MacKay, D. J. C },
}

@article{sinn_1,
  title={Accurate prediction of protein function using statistics-informed graph networks},
  author={Jang, Y. J. and Qin, Q.-Q. and Huang, S.-Y. and Peter, A. T. J. and Ding, X.-M. and Kornmann, B.},
  journal={Nature communications},
  volume={15},
  number={1},
  pages={6601},
  year={2024},
  publisher={Nature Publishing Group UK London}
}

@article{sinn_2,
  title={Learning stochastic dynamics with statistics-informed neural network},
  author={Zhu, Y. and Tang, Y.-H. and Kim, C.},
  journal={Journal of Computational Physics},
  volume={474},
  pages={111819},
  year={2023},
  publisher={Elsevier}
}

@Article{gmd-14-3007-2021,
AUTHOR = {Leach, N. J. and Jenkins, S. and Nicholls, Z. and Smith, C. J. and Lynch, J. and Cain, M. and Walsh, T. and Wu, B. and Tsutsui, J. and Allen, M. R.},
TITLE = {FaIRv2.0.0: a generalized impulse response model for climate uncertainty and future scenario exploration},
JOURNAL = {Geoscientific Model Development},
VOLUME = {14},
YEAR = {2021},
NUMBER = {5},
PAGES = {3007--3036},
URL = {https://gmd.copernicus.org/articles/14/3007/2021/},
DOI = {10.5194/gmd-14-3007-2021}
}

@article{smith2018fair,
  title={{FaIR} v1.3: a simple emissions-based impulse response and carbon cycle model},
  author={Smith, C. J. and Forster, P. M. and Allen, M. and Leach, N. and Millar, R. J. and Passerello, G. A. and Regayre, L. A.},
  journal={Geoscientific Model Development},
  volume={11},
  number={6},
  pages={2273--2297},
  year={2018},
  publisher={Copernicus GmbH}
}

@article{weichel2024uncertaintyquantificationportfoliotemperature,
      title={Uncertainty Quantification in Portfolio Temperature Alignment}, 
      author={Weichel, H. and Zinovev, A. and Haario, H. and Simon, M.},
      year={2024},
      eprint={2412.14182},
      archivePrefix={arXiv},
      primaryClass={q-fin.PM},
        journal={arXiv preprint arXiv:2412.14182},
}

@article{gobetMetamodellingPathsSimple2025,
  title={Meta-modelling paths of simple climate models using neural networks and dirichlet polynomials: an application to DICE},
  author={Gobet, E. and Liu, Y. and Vermandel, G.},
  journal={European Actuarial Journal},
  pages={1--45},
  year={2025},
  publisher={Springer}
}

@article{haarioDRAMEfficientAdaptive2006,
  title = {{{DRAM}}: {{Efficient}} Adaptive {{MCMC}}},
  author = {Haario, H. and Laine, M. and Mira, A. and Saksman, E.},
  year = {2006},
  date = {2006-12-01},
  journal = {Statistics and Computing},
  shortjournal = {Statistics and Computing},
  volume = {16},
  number = {4},
  pages = {339--354},
  issn = {1573-1375},
  doi = {10.1007/s11222-006-9438-0},
  url = {https://doi.org/10.1007/s11222-006-9438-0},
}

@book{ProcessAnalysisHimmelblau,
    author = {Himmelblau, D. M.},
    title = {Process analysis by statistical methods},
    publisher = {John Wiley \& Sons, Inc.},
    year = 1970,
    note = {Exercise 9.9}
}

@article{o2014new,
  title={A new scenario framework for climate change research: the concept of shared socioeconomic pathways},
  author={O’Neill, B. C. and Kriegler, E. and Riahi, K. and Ebi, K. L. and Hallegatte, S. and Carter, T. R. and Mathur, R. and Van Vuuren, D. P.},
  journal={Climatic change},
  volume={122},
  pages={387--400},
  year={2014},
  publisher={Springer}
}

@article{bourgey2024efficient,
  title={An Efficient {SSP}-based Methodology for Assessing Climate Risks of a Large Credit Portfolio},
  author={Bourgey, F. and Gobet, E. and Jiao, Y.},
  year={2024},
  journal={Available at hal-04665712}
}

@article{guerra2025learning,
  title={Learning Where to Learn: Training Distribution Selection for Provable OOD Performance},
  author={Guerra, N. and Nelsen, N. H. and Yang, Y.},
  journal={arXiv preprint arXiv:2505.21626},
  year={2025}
}

@book{villani2009optimal,
  title={Optimal transport: old and new},
  author={Villani, C{\'e}dric and others},
  volume={338},
  year={2009},
  publisher={Springer}
}

@book{santabrogio2015optimal,
  title={Optimal Transport for Applied Mathematicians: Calculus of Variations, PDEs, and Modeling},
  author={Santambrogio, F.},
  volume={87},
  year={2015},
  publisher={Birkhäuser/Springer}
}

@inproceedings{bolley2005weighted,
  title={Weighted Csisz{\'a}r-Kullback-Pinsker inequalities and applications to transportation inequalities},
  author={Bolley, F. and Villani, C.},
  booktitle={Annales de la Facult{\'e} des sciences de Toulouse: Math{\'e}matiques},
  volume={14},
  number={3},
  pages={331--352},
  year={2005}
}

@article{raissi2019physics,
  title={Physics-informed neural networks: A deep learning framework for solving forward and inverse problems involving nonlinear partial differential equations},
  author={Raissi, M. and Perdikaris, P. and Karniadakis, G. E.},
  journal={Journal of Computational physics},
  volume={378},
  pages={686--707},
  year={2019},
  publisher={Elsevier}
}

@article{yang2021b,
  title={B-PINNs: Bayesian physics-informed neural networks for forward and inverse PDE problems with noisy data},
  author={Yang, L. and Meng, X. and Karniadakis, G. E.},
  journal={Journal of Computational Physics},
  volume={425},
  pages={109913},
  year={2021},
  publisher={Elsevier}
}

\end{document}